\documentclass[journal, twocolumn]{IEEEtran}
\IEEEoverridecommandlockouts
\usepackage{subfig}
\usepackage{bm}
\usepackage{epsfig}
\usepackage{graphicx}
\usepackage{epstopdf}
\usepackage{indentfirst}
\usepackage{amsmath,amssymb,amsfonts}
\usepackage{booktabs}
\usepackage{array}
\usepackage{multirow}
\usepackage{dsfont}
\usepackage{subcaption}
\usepackage{caption}
\usepackage{cases}
\usepackage{amsthm}
\newtheorem{theorem}{Theorem}
\newtheorem{lemma}{Lemma}
\usepackage{times}
\usepackage[linesnumbered,ruled,vlined]{algorithm2e}
\usepackage{setspace}
\usepackage{graphicx}

\usepackage[table,xcdraw]{xcolor}
\usepackage{enumitem}

\begin{document}
\title{\huge {pFedWN: A Personalized Federated Learning Framework for D2D Wireless Networks with Heterogeneous Data}}

 \author{
     \IEEEauthorblockN{Zhou Ni, Masoud Ghazikor, Morteza Hashemi} \\
 \IEEEauthorblockA{Department of Electrical Engineering and Computer Science, University of Kansas, Lawrence, USA}

 \IEEEauthorblockA{Email: \{zhou.ni, masoudghazikor, mhashemi\}@ku.edu}
 }

\maketitle
\setlength{\parskip}{0pt}
\thispagestyle{plain}
\pagestyle{plain}

\begin{abstract}
Traditional Federated Learning (FL) approaches often struggle with data heterogeneity across clients, leading to suboptimal model performance for individual clients. To address this issue, Personalized Federated Learning (PFL) emerges as a solution to the challenges posed by non-independent and identically distributed \emph{(non-IID)} and \emph{unbalanced data} across clients. Furthermore, in most existing decentralized machine learning works, a perfect communication channel is considered for model parameter transmission between clients and servers. However,  decentralized PFL over wireless links introduces new challenges, such as resource allocation and interference management.  To overcome these challenges, we
formulate a joint optimization problem that incorporates the underlying device-to-device (D2D) wireless channel conditions into a server-free PFL approach. The proposed method, dubbed pFedWN, optimizes the learning performance for each client while accounting for the variability in D2D wireless channels. To tackle the formulated problem, we divide it into two sub-problems: PFL neighbor selection and PFL weight assignment. The PFL neighbor selection is addressed through channel-aware neighbor selection within unlicensed spectrum bands such as ISM bands. Next, to assign PFL weights, we utilize the Expectation-Maximization (EM) method to evaluate the similarity between clients' data and obtain optimal weight distribution among the chosen PFL neighbors. Empirical results show that pFedWN provides efficient and personalized learning performance with non-IID and unbalanced datasets. Furthermore, it outperforms the existing FL and PFL methods in terms of learning efficacy and robustness, particularly under dynamic and unpredictable wireless channel conditions.

\end{abstract}


\section{Introduction}
As machine learning technologies continue to advance, distributed machine learning has emerged as a popular and practical approach to improve both the learning efficiency and accuracy of model training. FL \cite{mcmahan2017communication}, a key form of distributed learning, allows users to collaborate on a global model without sharing their raw data, thus maintaining privacy during the training process. However, traditional FL approaches face challenges, particularly when dealing with \emph{non-IID and unbalanced 
data} across users. 
For instance, our representative numerical results in Fig.~\ref{Fig:locvsglob} illustrate that applying the FedAvg algorithm under a non-IID and unbalanced data distribution across FL clients leads to a notable disparity: the global model performs poorly on a target client dataset, even though the FL network achieves relatively higher average test accuracy across all FL clients.
To address this problem, PFL has been developed as a potential solution \cite{fallah2020personalized}. PFL adapts the learning process by using local data to refine model performance and tailor it to individual datasets and metrics.
For instance, the authors in \cite{deng2020adaptive} propose an adaptive PFL algorithm that allows each client to train their local models while contributing to the global model. This approach aims to find an optimal balance between local and global models to enhance personalization and efficiency in FL.

\begin{figure}[t]
\centering
\includegraphics[width=\linewidth, trim={60 10 110 50}, clip]{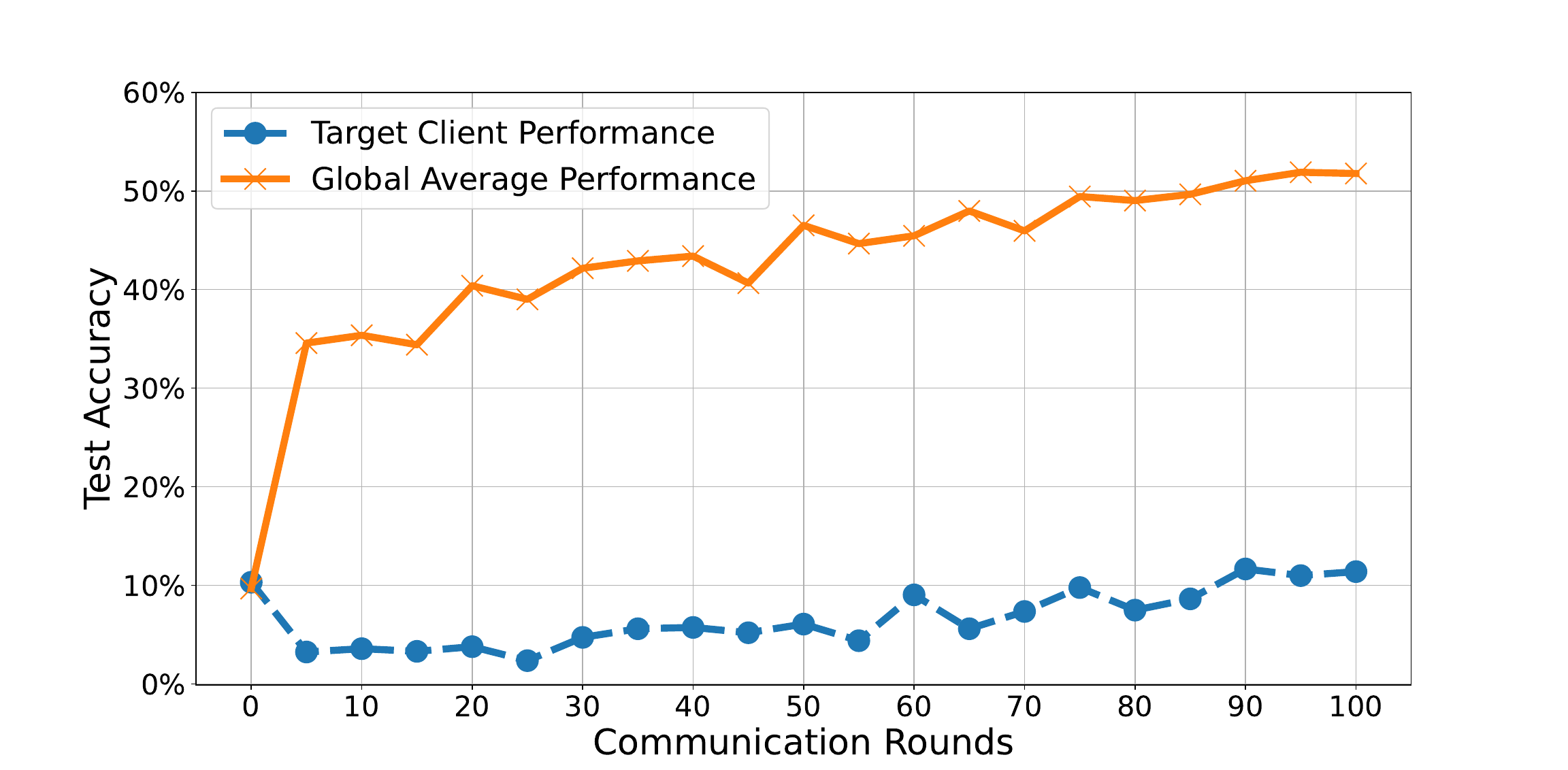}
\caption{\small{Performance comparison of a typical target client vs. the global model performance (FedAvg) for FL with 11 clients with non-IID and unbalanced CIFAR-10 dataset. }}
\label{Fig:locvsglob}
\end{figure}

Building on the foundations of PFL, another approach to achieve personalization is to collaborate only with those clients who have similar datasets in terms of distribution and size. 
 Therefore, using these approaches, users aim to find the appropriate neighbors with similar datasets to improve their model performance in PFL~\cite{donahue2021optimality}. 
However, this raises an important question: \emph{How to identify those neighbors with similar data structures to form personalized federations?} To address this question, there have been several prior works~\cite{donahue2021optimality,ghosh2020efficient,werner2022towards,sun2024collaborate} focused on efficient algorithms to cluster clients based on the similarity of their data distributions. 
Nevertheless, these research studies typically consider ideal and perfect transmission channels for establishing collaboration with neighboring clients within a cluster.  Unreliable or slow wireless links can lead to outdated or incomplete model updates, thereby significantly impacting the learning model's convergence rate and accuracy. 

In this context, it is crucial to consider the impact of wireless channel conditions and \emph{over-the-air} collaboration between neighboring clients.
  Over-the-air and wireless PFL algorithms should account for transmission error probabilities, constrained transmission power, interference from other clients, etc. 
  Although 
a multitude of prior works have studied combining wireless communication and FL (see, for example, ~\cite{chen2020wireless, amiri2020federated,chen2020joint,wang2022interference,wang2023performance,yan2024performance,chen2024robust,chen2024optimizing}),
to the best of our knowledge, joint optimization of wireless communication and \emph{personalized} FL with data heterogeneity have not been fully investigated in the context of interference-limited D2D wireless networks operating in unlicensed spectrum bands. 


In this paper, we formulate an optimization problem to jointly consider the data similarity across different clients and underlying wireless channel conditions. We propose pFedWN that 
considers the data structure of the clients on the learning side by using the EM approach. Furthermore, we integrate the state of wireless channel conditions into the decision-making process for the neighbor selection. For a target client in the D2D network, two subproblems need to be addressed before the PFL training. Those are: (i) which neighbors should be selected as PFL neighbors among all neighbors, and (ii) how much this target client wants to learn from each of those PFL neighbors' models. In our system model, the neighbor selection depends on the wireless channel condition, which defines the link transmission error. In addition, the EM algorithm determines the optimal aggregation weights for the target client to achieve personalization. 
In
summary, our main contributions  are as follows:
\begin{itemize}
    \item We propose a PFL framework for D2D wireless links and formulate a joint optimization problem that simultaneously accounts for client data similarity and the dynamics of wireless channel conditions. The two-step solution, named pFedWN, is designed to maximize the learning performance of the target client while incorporating the wireless channel conditions. 
    \item For the first subproblem, we propose a channel-aware PFL neighbor selection policy to enhance the PFL communication efficiency.
    The solution is based on calculating the probability of transmission error according to the signal-to-interference-plus-noise ratio (SINR) between the neighbors and the target client. 
    \item For the second subproblem, we employ the EM approach to estimate the similarity in data structure between the target client and its chosen neighbors, enabling us to determine the optimal aggregation weights for each neighbor. With the model exchange between those PFL clients, the EM algorithm allows us to accurately estimate data similarity iteratively without direct access to the client's raw data. This approach combines the target client's local model and its neighbors' model information for the next iteration of model updates. 
    \item Through numerical results, we demonstrate that the proposed pFedWN algorithm achieves a better learning performance compared to baseline models in non-IID and unbalanced dataset scenarios. This shows the potential of our approach in enhancing the efficiency and effectiveness of PFL in D2D communication networks.
\end{itemize}


\noindent 
\textbf{Organizations.}
We describe the related works in Section \ref{sec:relatedworks}. In Section \ref{sec:problem}, we introduce the system model, followed by the problem formulation and proposed PFL approach in Section \ref{sec:problem-formulation}.
Section \ref{sec:results} presents our numerical results, and Section \ref{sec:concludes} concludes the paper. 

\section{Related Works}
\label{sec:relatedworks}
\textbf{Personalized Federated Learning.} To tackle the diverse data heterogeneity across clients in FL, PFL frameworks have been introduced. For instance, Per-FedAvg \cite{fallah2020personalized} seeks to optimize a personalized model for each client by considering the model's initial parameters during the meta-learning process. pFedMe \cite{t2020personalized} employs a personalized model for each client by incorporating more client-specific updates and a global model update to capture the global data distribution by applying the EM approach. MOCHA \cite{smith2017federated} is designed to handle multi-task learning in federated settings, allowing for personalized models considering network size and node heterogeneity. APFL \cite{deng2020adaptive} proposes a method that adaptively updates each client's local model in tandem with the global model, which introduces a mechanism to control the influence of the global model on the local updates, providing a fine-tuned balance that can be adjusted according to the level of personalization required for each client's tasks.
\cite{yang2023dynamic} introduces FedDPA that integrates Dynamic Fisher Personalization (DFP) and Adaptive Constraint (AC) techniques to address the challenges of inflexible personalization and convergence issues in PFL under differential privacy constraints.

In parallel, cluster-based PFL algorithms have been developed. For instance, the authors in \cite{ghosh2020efficient} proposed a framework to group devices or clients into clusters based on the similarity of their data distributions. {Additionally,~\cite{werner2022towards} introduces a novel approach to PFL by clustering clients with similar objectives using a threshold-based algorithm and local client momentum, achieving optimal convergence rates comparable to those possible with known client clustering configurations.} In~\cite{ni2023efficient}, the authors developed a dynamic FL cluster selection algorithm using multi-armed bandit methodologies.

While these prior works have made significant contributions to PFL algorithms, our work advances these studies by (i) integrating unreliable wireless channel models into PFL algorithms and (ii) introducing a novel EM-based approach to determine the optimal ``collaboration weights'' for neighboring nodes.

\textbf{Federated Learning over Wireless Networks.} 
There has been an extended amount of work on considering Federated learning with wireless communications. 
For example, \cite{wang2022interference} tackles the problem of interference in multi-cell wireless networks for over-the-air federated learning, proposing effective management strategies to enhance learning efficiency in cooperative multi-cell FL optimization. Additionally, different resource allocation problems combined with FL have been extensively investigated in several recent works~\cite{ shi2020joint, dinh2020federated, xu2020client, salh2023energy, zhou2023resource,yang2020energy}. 
In another line of work, 
\cite{al2023edge} integrates FL with edge computing in the context of upcoming 6G networks to achieve ubiquitous data-driven services while also addressing the challenges of communication costs and data privacy.  The authors in~\cite{pei2023federated} propose a ``super-wireless-over-the-air'' federated learning framework to overcome the limitations of 5G for IoT, highlighting improvements in privacy and data transmission efficiency. In \cite{al2023decentralized}, the authors present a decentralized aggregation framework for energy-efficient FL via D2D communications. 
In addition, \cite{parasnis2023connectivity} introduces a semi-decentralized FL technique that leverages the properties of time-varying D2D networks to optimize the trade-off between convergence rate and communication costs. 

Furthermore, user selection for collaboration in FL is yet another key research problem~\cite{fu2023client}.
For example, the authors in~\cite{xu2023energy} propose a framework to optimize user selection and aggregator placement in mobile edge computing networks to minimize energy consumption while preserving learning accuracy.
Furthermore, a trust evaluation mechanism for federated learning in Digital Twin Mobile Networks (DTMN) was proposed in \cite{guo2023tfl}, which focuses on ensuring the reliability of the model by taking into account direct trust evidence and recommended trust information.
Additionally, \cite{chen2020joint} addresses the challenge of implementing FL over wireless networks by focusing on optimizing user selection and resource allocation to improve the accuracy of the global FL model.  
Moreover, the authors in \cite{chen2020convergence} show the impacts of user selection in FL over wireless networks in terms of learning performance and convergence time. 
To optimize the convergence time and performance of FL, the paper proposes a probabilistic user selection strategy that prioritizes users with the most significant impact on the global FL model. 

Therefore, we highlight that significant amounts of work have explored FL over wireless networks, some of which consider user selection as well (e.g., \cite{chen2020joint}, \cite{chen2020convergence}).  However, these works primarily focus on IID data distributions. In contrast, we focus on \emph{personalized FL} over wireless networks where clients possess \emph{unbalanced and non-IID datasets}. To this end, we integrate individual client learning behaviors with the constraints of wireless networks. 
That is, under the limitation of wireless communication resources, clients must act in their self-interest and make strategic decisions about collaborating with neighbors to improve their learning outcomes. By considering both personalized learning performance as well as the underlying dynamic wireless channel conditions, we formulate a joint optimization problem and propose the pFedWN algorithm, which not only improves the efficiency of the learning process by optimizing collaborating neighbors but also enhances the reliability of communications by selecting neighbors with better channel conditions. 

In addition, we consider D2D communications to enable clients to directly exchange and aggregate models based on their local data profiles without a central server. 
The decentralized nature of D2D communications enhances the robustness of the learning network and reduces the risk of server-centric failures by improving the system's resilience to disruptions. 
This approach aligns with the increasing need for localized, real-time data processing in applications such as autonomous vehicles and mobile healthcare, where server-based delays and bandwidth limitations could impact wireless communication and learning performance.


\begin{figure}[t]
\centering
\includegraphics[width=\linewidth]{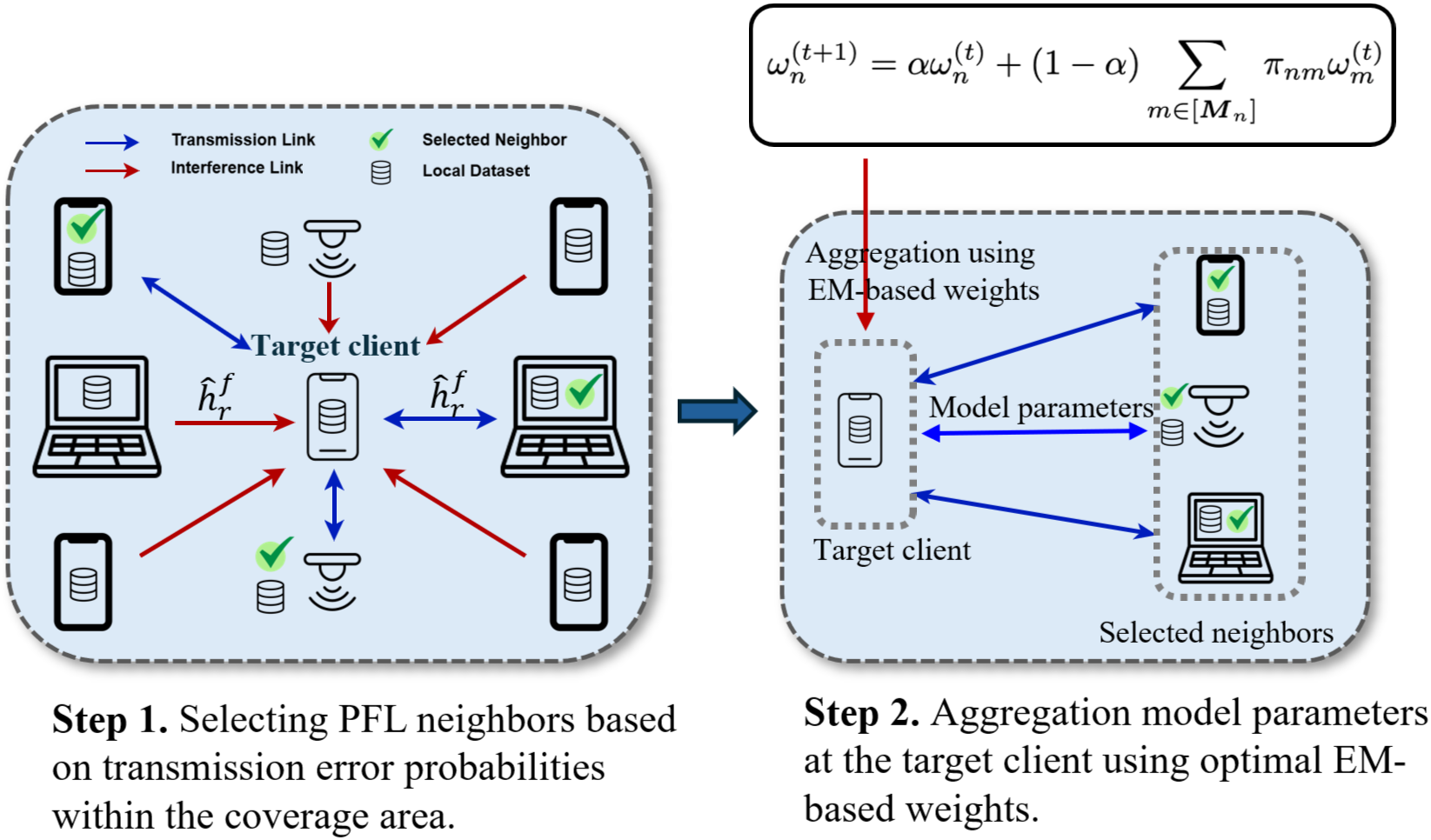}
\caption{\small{In the proposed system model, the target client selects a set of neighbors based on the wireless channel conditions. Then, the target client assigns model aggregation weights.} }
\label{Fig:sys}
\end{figure}
\section{System Model}
\label{sec:problem}

\subsection{D2D Learning Model}
\label{subsec:model}
We consider a PFL framework within a server-free network, where $[\boldsymbol{N}]:= \{1, 2, ..., N\}$ denotes the set of clients (or nodes).  
Furthermore, the $n$-th target client is surrounded by $g_n$ neighbors denoted by the set $[\boldsymbol{G}_n]$, such that $|[\boldsymbol{G}_n]| = g_n$ which $|.|$ indicates the cardinality of a set.
We assume that the $n$-th client possesses a dataset $D_{n} = { (\textbf{x}_{i}^{(n)}, y_{i}^{(n)}) }_{i = 1}^{k_n}$, drawn from a distribution $\mathcal{D}_{n}$ over the space $\mathcal{X} \times \mathcal{Y}$, and the total dataset size across all clients is denoted by $K = \sum_{n = 1}^{N} k_n$.


In a typical FL scenario, the clients' local learning models are aggregated at the central server to obtain a global model. However, PFL algorithms are focused more on clients' models based on their data profiles. The target client not only updates its local model based on its own data but also adaptively incorporates model updates from a set of selected neighbors within the network. 
In our server-free PFL framework, as shown in Fig.~\ref{Fig:sys}, the target client first selects a set of neighbors based on the wireless channel conditions. Then, the target client optimally assigns weights for the model parameters received from the selected neighbors. Finally, the model aggregation for the $n$-th target client with its selected neighbors takes place on the client side without any central server. Thus, the local model updating for the $n$-th client can be written as: 
\begin{equation} \label{eq:5}
    \omega_{n}^{(t+1)} = \alpha \omega_{n}^{(t)} + (1-\alpha)\sum_{m \in {[\boldsymbol{M}_n]}}\pi_{nm}\omega_{m}^{(t)},
\end{equation}
where $[\boldsymbol{M}_n]$
is a set of neighbors selected by the $n$-th client such that $|[\boldsymbol{M}_n]| = M_n$, and $\omega_{m}^{(t)}$ is the  model parameters from the $m$-th selected neighbor at time slot $t$. Furthermore, $\omega_{n}^{(t)}$ is obtained from client $n$ by training its local dataset using the standard gradient descent method, i.e.,:
\begin{equation} \label{eq:2}
    \omega^{(t, j+1)}_{n} = \omega^{(t, j)}_{n} - \eta \nabla f_{n}(\omega_{n}^{(t, j)}),
\end{equation}
where $j \in \{0,..., E - 1\}$ and $E$ is the total number of local iterations. Additionally, after $E$ local training on target client $n$, we take the local model at the final iteration $\omega^{(t, E)}_{n}$ to be the model for the next round aggregation, i.e., $\omega^{(t)}_{n} \triangleq \omega^{(t, E)}_{n}$.

To incorporate the model parameters of the selected neighbors, we aim to find the optimal weight assignment $\pi_{nm}$, which represents how much the target client $n$ wants to learn from its $m$-th neighbor. In addition, $\alpha$ is a hyperparameter that depends on how much the client wants to learn from all its neighbors compared to its current local model parameters. 

\noindent 
\subsection{{Channel-Aware PFL Neighbor Selection}} Given the described D2D learning model, in this section, we present the underlying communication setup for model parameter exchange between the target client and neighbors.
Our goal is to select a subset $\boldsymbol{M}_n$ neighbors out of $\boldsymbol{G}_n$ 
to be the PFL neighbor for the target client $n$. In this paper, we use transmission error probability as a metric for selecting those PFL neighbors. 

\textbf{Channel Modeling and Transmission Error Probability.}
We assume that all clients (wireless nodes) operate in an unlicensed spectrum band that is partitioned into $|\boldsymbol{F}|$ sub-channels. 
Furthermore, $\boldsymbol{S}$ represents the set of communication sessions sharing the same spectrum band, and $s \in \boldsymbol{S}$ denotes the individual session between the selected neighbor and the target client. Also, \( r \in \boldsymbol{S}\setminus s\) represents the session between a specific interfering neighbor and the target client, while the set \( \boldsymbol{R} \) denotes all interfering neighbors.

Given the transmit power $P_s$, the received power at the target client is given by $P_{rx} = P_s |h_s^f|^2$ where $h_s^f = \tilde{h}_s^f \hat{h}_s^f$ denotes the channel gain for sub-channel $f \in \boldsymbol{F}$. Here, $\tilde{h}_s^f$ and $\hat{h}_s^f$ are the channel fading coefficient and the square root of the path loss, respectively. Using the single-slope path loss model, $\hat{h}_s^f$ is defined as~\cite{Channel_Ghazikor_2024}:
\begin{equation}
\hat{h}_s^f = \frac{\lambda}{4 \pi d_0}\sqrt{(\frac{d_0}{d})^{\alpha_s}} \quad \text{if} \enskip d \ge d_0, 
\end{equation}
where $\alpha_s$, $d_0$, and $\lambda$ are the path loss exponent, reference distance, and signal wavelength, respectively.

We assume a block fading channel model in which $\tilde{h}_s^f$ follows the {Rayleigh distribution} corresponding to Non-Line-of-Sight (NLoS) channel. Accordingly, the probability density function (PDF) of the Rayleigh distribution is given by:
\begin{equation}
    \mathbf{P}(\tilde{h}_s^f = x) = \frac{2x}{\Gamma} e^{-\frac{x^2}{\Gamma}},
\end{equation}
where $\Gamma$ denotes the Rayleigh fading factor. Accordingly,
the selected neighbor transmits its packet to the target client over the best frequency channel $f^* = {\arg\max}_{f \in \boldsymbol{F}} \ \ \Tilde{h}_s^f \hat{h}_s^f$ if the channel fading coefficient is greater than the channel fading threshold $\beta_s$ ($\Tilde{h}_s^{f^*} \ge \beta_s > 0$) \cite{Guan-2016-ToTransmit}.
Therefore, each selected neighbor chooses the best sub-channel to transmit its packet in the presence of other neighbors.


 Let $\gamma_{th}$ represent the SINR threshold such that transmission error occurs if the SINR falls below this threshold. Therefore, the probability of transmission error between the neighbor and the target client due to the high impact of interference from interferer neighbors operating in the same frequency band can be expressed as follows:
\begin{equation*}
\begin{aligned} \label{Perr}
\mathbf{P}_s^{err}(\hat{\boldsymbol{h}}^f) & \triangleq \mathbf{P}(\gamma_s < \gamma_{th}) = \mathbf{P}\Biggl(\frac{P_s(h_s^f)^2}{\sigma^2+I_s^f(\hat{\boldsymbol{h}}^f_{-s})} < \gamma_{th} \Biggl),
\end{aligned}
\end{equation*}
where $P_s$ denotes the transmission power at the neighbor $s \in \boldsymbol{S}$, and $\sigma^2 = \kappa TW$ is the thermal noise power where $\kappa$, $T$, and $W$ are the Boltzmann constant, noise temperature, and bandwidth, respectively. Also, $I_s^f(\hat{\boldsymbol{h}}^f_{-s})$ represents the impact of interferer neighbors on the target client in which $(\hat{\boldsymbol{h}}^f_{-s}) \triangleq (\hat{\boldsymbol{h}}^f_{r})_{r \in \boldsymbol{S} \backslash s}$. Thus, $I_s^f(\hat{\boldsymbol{h}}^f_{-s})$ is given by \cite{Guan-2016-ToTransmit}:
\begin{equation}
I_s^f(\hat{\boldsymbol{h}}^f_{-s}) = \sum_{r\in \boldsymbol{S} \backslash s} P_r(h_r^f)^2\alpha_r^f(\beta_r).
\end{equation}
Here, $\alpha_r^f(\beta_r)$ equals one if interferer neighbor $r$ transmits, and zero otherwise. Furthermore, $(h_r^f)^2$ indicates the channel gain between the interferer neighbor $r$ and the target client. 
Applying the classical stochastic geometry approach, the {Log-normal distribution} can be used to model interference. Hence, the PDF of $I_s^f(\hat{\boldsymbol{h}}^f_{-s})$ can be defined as \cite{Tian-2016-Interference}:
\begin{equation}
    \mathbf{P}(I_s^f(\hat{\boldsymbol{h}}^f_{-s}) = x) = \frac{\exp(-\frac{(\ln x-\mu(\hat{\boldsymbol{h}}^f_{-s}))^2}{2\sigma^2(\hat{\boldsymbol{h}}^f_{-s})})}{x\sigma(\hat{\boldsymbol{h}}^f_{-s})\sqrt{2\pi}},
\end{equation}
where $\mu(\hat{\boldsymbol{h}}^f_{-s})$ and 
$\sigma(\hat{\boldsymbol{h}}^f_{-s})$ are the location and scale parameters in the Log-normal distribution, respectively, which are derived in Appendix \ref{log-normaldistribution}. Then, let $v_s(x,\hat{\boldsymbol{h}}^f_{-s})$ be the complementary cumulative distribution function (CCDF) of $I_s^f(\hat{\boldsymbol{h}}^f_{-s})$ which is given by:
\begin{align*}
v_s(x,\hat{\boldsymbol{h}}^f_{-s}) & \triangleq \mathbf{P}(x < I_s^f(\hat{\boldsymbol{h}}^f_{-s})) \\
 = \frac{1}{2} - & \frac{1}{2} \text{erf}(\frac{\ln x-\mu(\hat{\boldsymbol{h}}^f_{-s})}{\sqrt{2}\sigma(\hat{\boldsymbol{h}}^f_{-s})})  = 1-\phi(\frac{\ln x-\mu(\hat{\boldsymbol{h}}^f_{-s})}{\sigma(\hat{\boldsymbol{h}}^f_{-s})}),
\end{align*}
where $\text{erf}(x) = \frac{2}{\sqrt{\pi}} \int_0^x e^{-t^2}dt$ represents the error function and $\phi(x)$ denotes the cumulative distribution function (CDF) of the standard normal distribution. Finally, $\mathbf{P}_s^{err}(\hat{\boldsymbol{h}}^f)$ can be defined as follows:
\begin{equation*}
\begin{aligned}
\mathbf{P}_s^{err}(\hat{\boldsymbol{h}}^f) & = \mathbf{P}\Biggl(\frac{P_s(\hat{h}_s^f)^2(\Tilde{h}_s^f)^2}{\gamma_{th}}-\sigma^2 < I_s^f(\hat{\boldsymbol{h}}^f_{-s}) \Biggl) \\ 
& = \int_{\beta_s}^{\infty} \mathbf{P}(\Tilde{h}_s^f = x) v_s(\frac{P_s(\hat{h}_s^f)^2}{\gamma_{th}}x^2-\sigma^2,\hat{\boldsymbol{h}}^f_{-s})dx.
\end{aligned}
\end{equation*}
Therefore, given the calculated probability of transmission error $\mathbf{P}_s^{err}(\hat{\boldsymbol{h}}^f)$, the set of PFL neighbors $\boldsymbol{M}_n$ will be selected 
as follows: if $\mathbf{P}_s^{err}(\hat{\boldsymbol{h}}^f)$ for a specific neighbor $s \in \boldsymbol{G}_n$ is smaller than the threshold $\epsilon$, the neighbor $s$ would be chosen to be a PFL neighbor $s \in \boldsymbol{M}_n$.
It should be noted that the number of selected neighbors would be different for various $\gamma_{th}$. For example, if $\gamma_{th}$ increases, $\mathbf{P}_s^{err}(\hat{\boldsymbol{h}}^f)$ will increase since packets cannot be decoded correctly at the target client with low SINR. Thus, according to the selection policy, fewer neighbors with the same $\epsilon$ will be selected.


\section{Problem Formulation and Proposed Solution}
\label{sec:problem-formulation}
In this section, we formulate a joint optimization problem to integrate channel-aware PFL neighbor selection with an EM-based weight assignment to those selected neighbors. 
\subsection{PFL Problem Formulation}
Let $\textit{f}_{n}: \mathcal{X}\times \mathcal{Y} \to \mathbb{R}_{+}$ denote the loss function for the $n$-th client. 
According to the presented system model, the PFL problem can be formulated as follows: 
\begin{align}
\label{eq:1}
& \min_{\omega_{n}, \boldsymbol{M}_{n}, \boldsymbol{\pi}_{n}} \quad \frac{1}{N} \sum_{n=1}^{N} \mathbb{E}_{\textbf{x},y \sim \mathcal{D}_{n}} [f_{n} ( \omega_{n}, \boldsymbol{M}_{n}, \boldsymbol{\pi}_{n} ; \textbf{x}_{i}^{(n)}, y_{i}^{(n)})], 
\end{align}
where $\omega_{n}$ is the $n$-th target client's local model parameters and the loss function is denoted by $\mathbb{E}_{\textbf{x},y \sim \mathcal{D}_{n}} [f_{n}  (\omega_{n}, \boldsymbol{M}_{n}, \boldsymbol{\pi}_{n}; \textbf{x}_{i}^{(n)}, y_{i}^{(n)})]$ of model parameters $\omega_{n}$ given the data point $\textbf{x}_{i}^{(n)}$ and true label $y_{i}^{(n)}$. The loss function depends on the model parameters $\omega_n$,  set of selected neighbors $\boldsymbol{M}_n$, and the aggregation weights $\boldsymbol{\pi}_{n} = [\pi_{n1}, \pi_{n2}, ..., \pi_{nM_n}]$ assigned to the PFL neighbors in set $\boldsymbol{M}_n$. 
To solve the formulated optimization problem, our framework focuses on a specific target client $n$ within the network to enhance its learning model by leveraging collaborations with selected neighboring clients. To this end, we  
separate the original PFL problem into two sub-problems. First, we select the set of appropriate neighbors $\boldsymbol{M}_n$ out of $\boldsymbol{G}_n$ based on the {probability of transmission error}, which reflects the conditions of the wireless communication channel. As outlined in Section \ref{sec:problem}, the selected PFL neighbors \( \boldsymbol{M}_n \) are those whose transmission error to the target client \( n \) is less than \( \epsilon \). Secondly, the optimal aggregation weights $\boldsymbol{\pi}_n$ are calculated, which capture data and model similarity between the selected neighbors and target client. 


\subsection{Proposed PFL Weight Assignment Methodology}
\label{sec:pFedWN}
In this section, we leverage the EM algorithm to find the optimal aggregation weights $\boldsymbol{\pi}_n$. Once the PFL neighbors $\boldsymbol{M}_n$ are selected, they are not treated equally when their models are aggregated at the target client side. Instead, each selected PFL neighbor is assigned a specific learning weight. This weight is determined based on the similarity of their data to that of the target client; that is, neighbors with more similar data can contribute more meaningfully to the personalized learning model of the target client.
To capture the data similarity of datasets across the neighbor clients, we consider the local data distribution $\mathcal{D}_{n}$ of a target client $n$ not in isolation.
Rather, we assume that the target client's data distribution is hypothesized to be a mixture of the data distributions of its neighboring clients $\{ \mathcal{D}_m \}_{m = 1}^{M_{n}}$. The weights are defined as $\boldsymbol{\pi}_{n} = [\pi_{n1}, \pi_{n2}, ..., \pi_{nM_{n}}] \in \Delta^{M_{n}}$. Here, the simplex $\Delta^{M_{n}}$ represents a prior belief system about the association of data between the target client and its neighbors.

To capture the latent relationships between the target client and its PFL neighbors, we introduce a latent variable $z_{n}$ for each client $n$. 
Using the EM algorithm, we consider the loss function:
\begin{equation}
l(h_{\omega_{nm}^*}(\textbf{x}), y) = -\log(p_{nm}(y|\textbf{x})) + B,
\end{equation}
to assess the likelihood of the data given the model parameters, with $B$ as a normalization constant and the hypothesis $h$ with optimal model parameters. It should be noted that the goal of employing the EM algorithm in our framework is to find the optimal weights $\boldsymbol{\pi}_n$ that estimate the data distribution similarities between clients. To obtain the optimal weights for each PFL neighbor, the EM algorithm iteratively performs the following steps:
\begin{algorithm}[t]
\SetAlgoLined
\DontPrintSemicolon
\caption{\small PFL Neighbor Selection and Weight Assignments}
\label{alg:pFedWN}
\SetKwInput{KwInput}{Input}
\SetKwInput{KwOutput}{Output}
\KwInput{$\boldsymbol{G}_n$: Neighbors of target client $n$}
    \For{neighbor $s$ of target client $n$}{
        \If{$\mathbf{P}_s^{err}(\hat{\boldsymbol{h}}^f) < \epsilon$}{
            select the PFL neighbor $s$ and add it to $\boldsymbol{M}_n$
        }
    }
\KwOutput{$\boldsymbol{M}_n$}
\KwInput{Selected PFL neighbors $\boldsymbol{M}_n$, number of communication rounds $T$, Initial model parameters $\Omega^{(0)}$, Prior $\Pi$}
\For{$t = 1, \dots, T$}{
     
        \For{each neighbor $m$ in $\boldsymbol{M}_n$}{
        \textbf{E-Step:} 
            Compute $\lambda_{nm}^{(t)}$ using Eq. \eqref{eq:E-step}
                   
        \textbf{M-Step: }
            Update $\pi_{nm}^{(t)}$ using Eq.~\eqref{eq:M-step-1}; 
            Update $\omega_{nm}^{(t+1)}$ using Eq.~\eqref{eq:M-step-2}
        }
    \If{convergence criterion is met}{
        \KwRet $\pi^{*}_{nm}$\;
    }
}
\KwOutput{Optimal weights $\pi^*_{nm}$}
\label{alg:algorithm-1}
\end{algorithm}
\begin{algorithm}[t]
\SetAlgoLined
\DontPrintSemicolon
\caption{pFedWN}
\label{alg:pFedWN}
\SetKwInput{KwInput}{Input}
\SetKwInput{KwOutput}{Output}
\KwInput{Target client local models $\omega_{n}^{(0)}$ and selected PFL neighbor model $\omega_m^{(0)}$ for $m \in [\boldsymbol{M}_n]$. Assigned weights for $m$ neighbors $\pi_{nm}^{(*)}$}
\For{$t = 1, \ldots, T$}{
    \For{neighbor $m$ of client $n$}{
        $\omega_{m}^{(t)} = \omega_{m}^{(t-1)} - \eta \nabla f_m \left( \omega_{m}^{(t-1)}, \xi_{m}^{t} \right)$\;
        \If{D2D Communication = \textbf{True}}{
            each PFL neighbor $m$ sends $\omega_m^{(t)}$ to the target client $n$
        }
        
    }
    \If{Target Client $n$ = \textbf{True}}{
        \If{D2D Communication = \textbf{True}}{
            receives $\omega_{m}^{(t)}$ from its selected neighbors,\\
            model aggregation: $\omega_{n}^{(t+1)} = \alpha \omega_{n}^{(t)} + (1-\alpha)\sum_{m \in [\boldsymbol{M}_n]}\pi_{n m}^{(*)}\omega_{m}^{(t)}$\;
        }
        Target client local training: \\
        $\omega_{n}^{(t, j+1)} = \omega_{n}^{(t, j)} - \eta \nabla f_n \left( \omega_{n}^{(t, j)}, \xi_{n}^{t+1} \right)$\;     
    }
}
\KwOutput{Optimal model parameters $\omega^{*}_n$}
\label{alg:algorithm-2}
\end{algorithm}

\begin{itemize}[leftmargin = 4mm, itemsep = 0.01in, parsep = 0.005in, topsep = 0.005in]
\item \textbf{E-STEP:} It determines the posterior probabilities $p(z_{n} = m |D_{n})$, which represent the probability that the latent variable $z_n$ 
is equal to $m$, given the target client's data $D_n$. This reflects the likelihood of the target client's data originating from each neighbor under the current model estimates: 
\begin{equation}
    \lambda_{nm}^{(t+1)} = q^{(t+1)}(z_n = m) \propto \pi_{nm}^{(t)}\exp{[-l(h_{\omega_{nm}^{(t)}}(\textbf{x}_i^{(n)}), y_i^{(n)})]}.
    \label{eq:E-step}
\end{equation}

\item \textbf{M-STEP:} It updates the weights and model parameters based on the likelihoods from the E-step:
\begin{equation}
    \pi_{nm}^{(t+1)} = \frac{\sum_{i=1}^{k_n}\lambda_{nm}^{(t+1)}}{k_n},
    \label{eq:M-step-1}
\end{equation}
\begin{equation}
    \omega_{nm}^{(t+1)} \in \arg \min_{\omega}  \sum_{i=1}^{k_n} \lambda_{nm}^{(t+1)} l(h_{\omega}(\textbf{x}_i^{(n)}), y_i^{(n)}).
    \label{eq:M-step-2}
\end{equation}
\end{itemize}

\noindent 
The convergence of the EM approach is well-studied in \cite{marfoq2021federated,sui2022find}. When the selected neighbors and target client use SGD as their optimizer, the convergence rate of the EM algorithm satisfies $\mathcal{O}(\frac{1}{\sqrt{T}})$ with $T$ communication rounds \cite{marfoq2021federated}. 

After convergence, the EM algorithm provides the optimal weight ${\pi}_{nm}^*$, which contains the appropriate weights of the target client's data distribution to those of its selected neighbors. These weights $\boldsymbol{\pi}_n$ are then leveraged in the target client's local model aggregation phase. This approach ensures that the target client's local model benefits from the knowledge of its selected neighbors' data characteristics, as expressed by the weights, thus enhancing the personalization aspect of FL.

Algorithm $1$ summarizes the channel-aware PFL neighbor selection and EM-based weight assignment. Furthermore,  Algorithm $2$ presents the overall pFedWN algorithm, including the local training at the target client $n$. 
To graphically demonstrate the operation of our solution, two time slots of the algorithm are shown in Figure \ref{Fig:timeslot}. At the initial time ($t=0$), the target client evaluates the quality of the communication channel with other clients within its range. A threshold determines which nearby clients will participate in the collaborative learning process. The target client then estimates the data similarity of each selected neighbor by the EM algorithm to appropriately weigh each neighbor’s model. With this setup, the target client begins its PFL using the pFedWN algorithm. For further communication rounds, the target client will receive the model from selected neighbors and update its model by applying the pFedWN algorithm for its local training in the current time slot.



\subsection{Convergence Analysis}
In this section, we provide the convergence analysis of our pFedWN algorithm for both strong convex and non-convex loss functions. To start our analysis, we make the following assumptions that hold in both scenarios:

\begin{itemize}[leftmargin = 4mm, itemsep = 0.01in, parsep = 0.005in, topsep = 0.005in]
    \item \textbf{Assumption 1} ($L$-Lipschitz continuous). \textit{There exists a $L > 0$ such that we have: 
$
    || \nabla f_{n}(\textbf{x}) -  \nabla f_{n}(\textbf{y}) || \leq L||\textbf{x} - \textbf{y}||, \forall \textbf{x},\textbf{y} \in \mathbb{R}^d, \forall n \in [\boldsymbol{N}]. 
$}

\item \textbf{Assumption 2} (Bounded Variance). \textit{The variance of stochastic gradients computed at each local data batch is bounded, i.e.,} 
$
    \mathbb{E}[||\nabla f_{n}(\textbf{x}; \xi) - f_{n}(\textbf{x})||^{2}] \leq \delta^{2}, \forall n \in [\boldsymbol{N}]. 
$
\end{itemize}
\subsubsection{Strong Convex Loss Function}
We first introduce an additional assumption for the strong convex loss function scenario. 
\begin{itemize}[leftmargin = 4mm, itemsep = 0.01in, parsep = 0.005in, topsep = 0.005in]

\item \textbf{Assumption 3} (Strong Convex). \textit{There exists a $\mu > 0$ such that
$
    f_{n}(\textbf{x}) \geq f_{n}(\textbf{y}) + \langle \nabla f_{n}(\textbf{y}) , \textbf{y} - \textbf{x}\rangle + \frac{\mu}{2}||\textbf{y} - \textbf{x}||^{2},  \forall \textbf{x},\textbf{y} \in \mathbb{R}^d, \forall n \in [\boldsymbol{N}]. 
$}
\end{itemize}




\begin{figure}[t]
\centering
\includegraphics[width=\linewidth, trim={70 20 50 20}, clip]{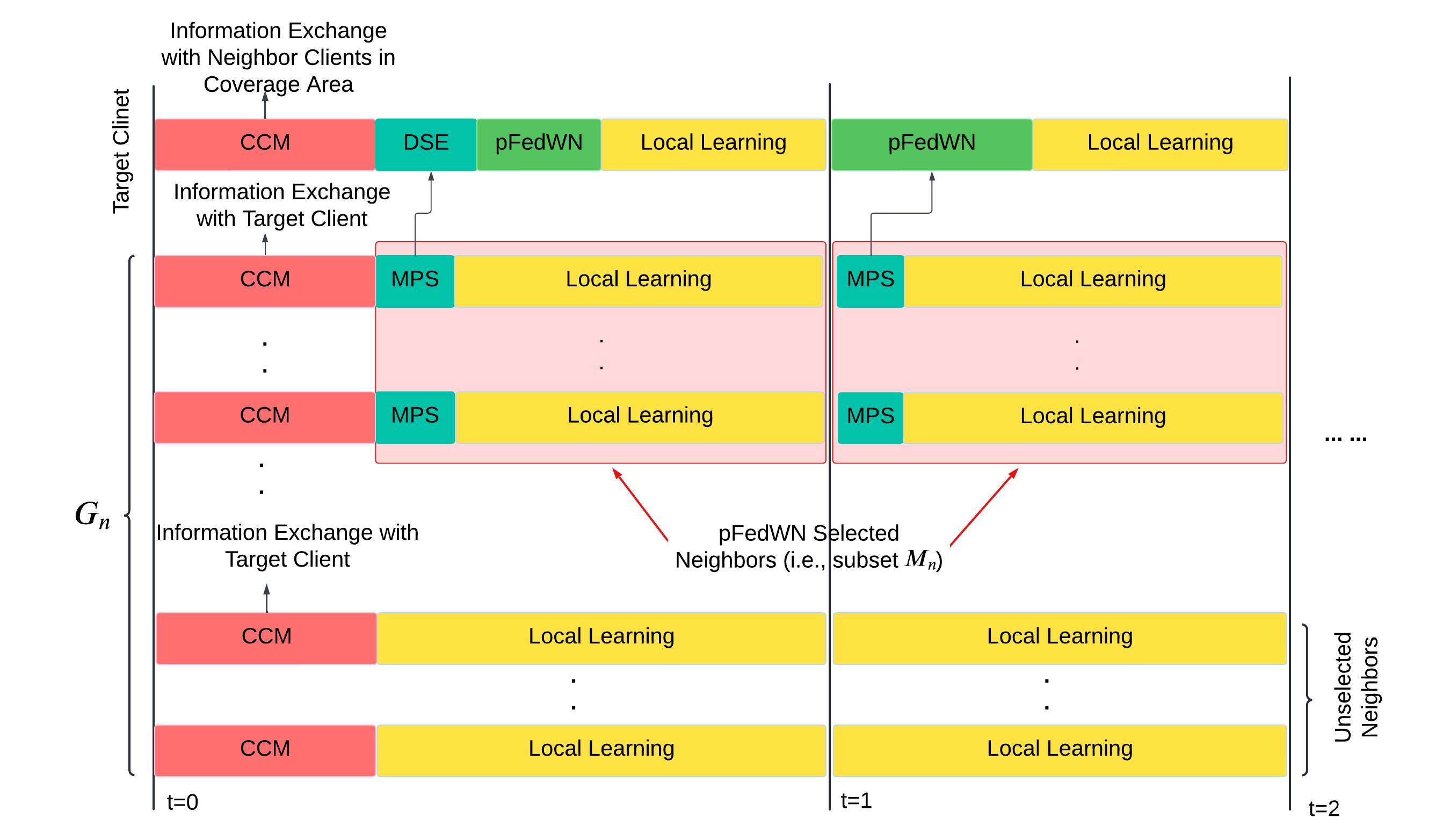}
\vspace{-0.5cm}
\caption{\small Example time slot for clients in the network. CCM: Communication Channel Measurement; DSE: Data Similarity Estimation; MPS: Model Parameters Sharing.}
\label{Fig:timeslot}
\end{figure}

\noindent 
These assumptions are prevalent in previous work~\cite{marfoq2021federated, parasnis2023connectivity} to make the convergence analysis tractable. 
We now prove the convergence of the model update $\omega_{n}^{(t+1)}$. 
Recall in Eq.~\eqref{eq:5}, $\alpha \in [0,1]$, and for any given $n$, $\sum_{m}\pi_{nm} = 1$. From Algorithm \ref{alg:pFedWN}, for the local model update for the selected neighbor $m \in [\boldsymbol{M}_{n}]$, we have: 
\begin{equation}\label{eq:12}
    \omega_{m}^{(t)} = \omega_{m}^{(t-1)} - \eta \nabla f_m \left( \omega_{m}^{(t-1)}, \xi_{m}^{t} \right). 
\end{equation}
We also assume that the local models $\omega_{m}$ are bounded. Given these assumptions, we present the following theorem. 
\begin{theorem}
Under Assumptions 1-3, and given the local model update rule in Eq.~\eqref{eq:12} for the selected neighbors $m \in [\boldsymbol{M}_{n}]$, and assuming that the local models $\omega_{m}$ are bounded, the pFedWN algorithm converges at a rate of $\mathcal{O}(\gamma^T)$ if $\alpha^{2}(2-\alpha)(1-\eta\mu)^{E} \leq 1$ holds. 
\end{theorem}
\begin{proof}
    The proof is provided in Appendix~\ref{section-appendix-thm-1}. 
\end{proof}
Next, we relax the assumption of a strongly convex loss function and present the convergence results of pFedWN for non-convex loss functions.

\subsubsection{Convergence Analysis with Non-Convex Loss Function}
In this part, we introduce the following additional assumptions.
\begin{itemize}[leftmargin = 4mm, itemsep = 0.01in, parsep = 0.005in, topsep = 0.005in]
\item \textbf{Assumption 4.} The expected squared distance between selected neighbor models and the target client's model is bounded:
$
 \mathbb{E}[||\omega_{m}^{t} - \omega_{n}^{t,E}||^{2}] \leq C, \forall m \in [\boldsymbol{M}_n], \forall t.
$

\item \textbf{Assumption 5.} The expected squared norm of the gradients is bounded:
$
 \mathbb{E}[||\nabla f_{n}(\omega_{n}^{(t)})||^{2}] \leq G^{2},  \forall n \in [\boldsymbol{N}], \forall t.
$
\end{itemize}
\noindent 

\begin{figure*}[t]
    \centering
    \begin{minipage}{.33\linewidth}      \includegraphics[width=6cm, trim={60 0 100 5}, clip]{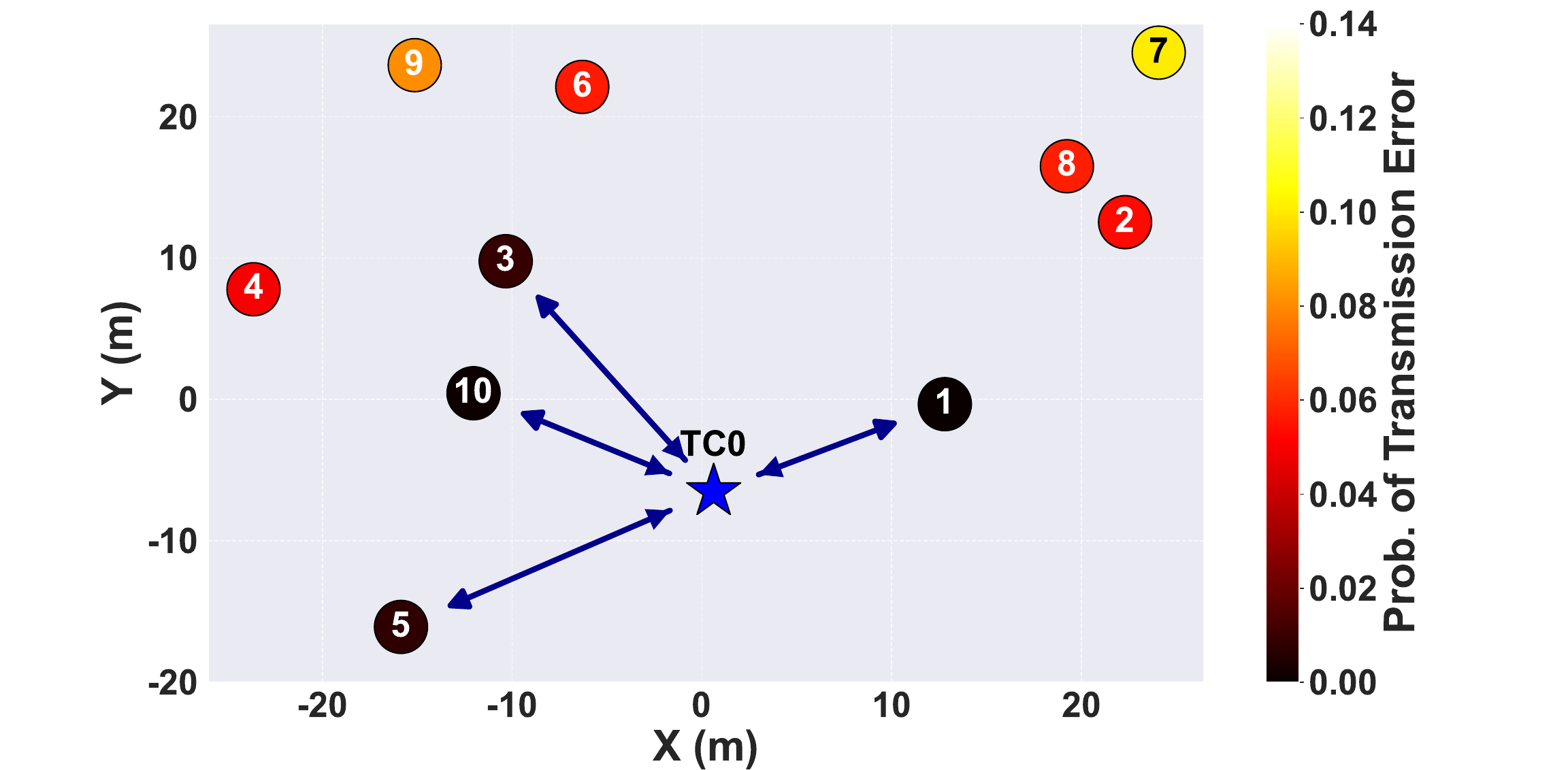}
        \centering (a) $\mathbf{P}_{s}^{err}$ heatmap for Case 1 ($\gamma_{th} = 5$)
    \end{minipage}%
        \begin{minipage}{.33\linewidth}      \includegraphics[width=6cm, trim={60 0 100 5}, clip]{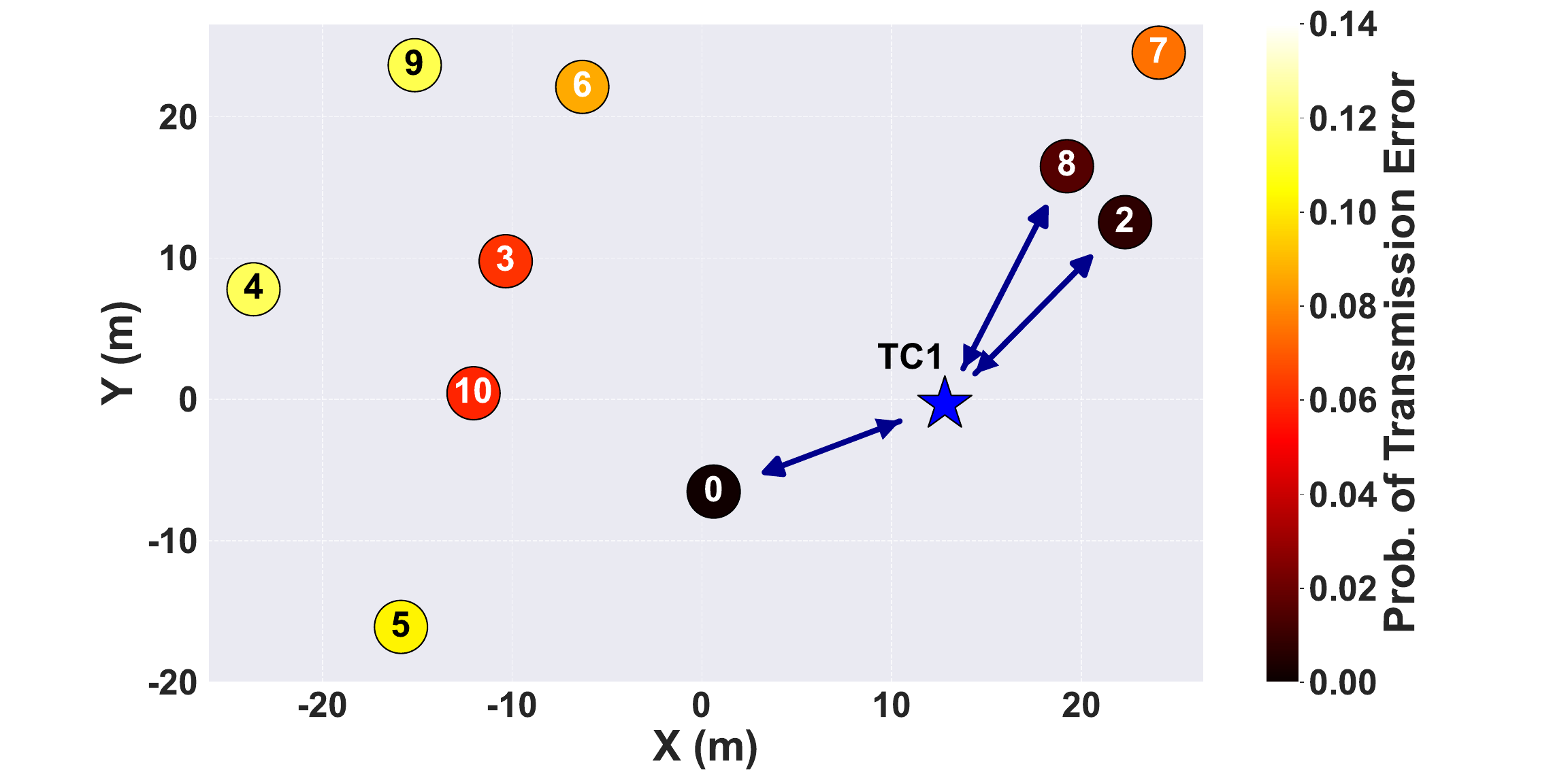}
        \centering (b) $\mathbf{P}_{s}^{err}$ heatmap for Case 2 ($\gamma_{th} = 10$)
    \end{minipage}%
    \begin{minipage}{.33\linewidth}      \includegraphics[width=6cm, trim={60 0 100 5}, clip]{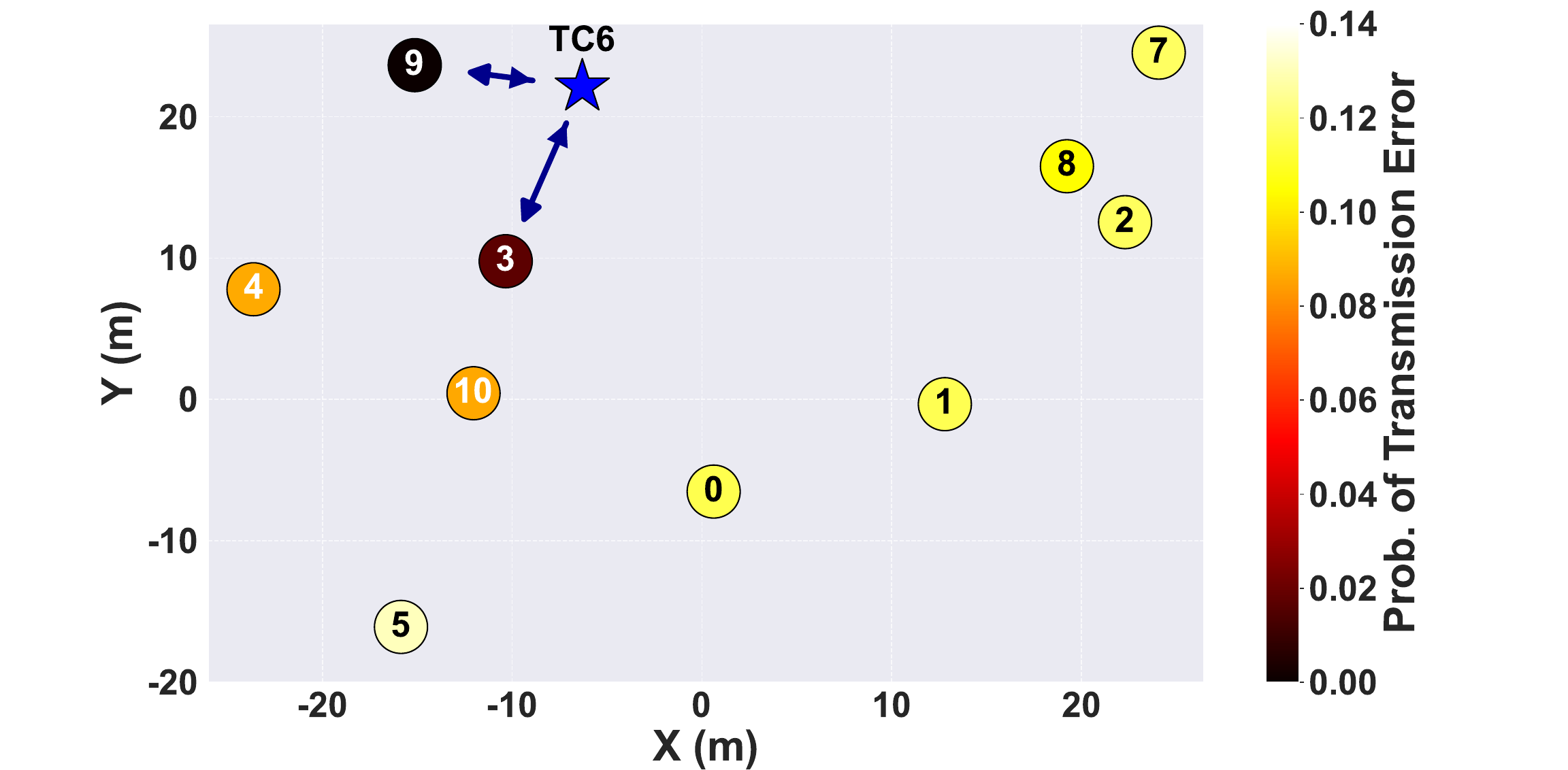}
        \centering (c) $\mathbf{P}_{s}^{err}$ heatmap for Case 3 ($\gamma_{th} = 15$)
    \end{minipage}%
    \caption{\small PFL-selected neighbors (SNs) as a function of transmission error probabilities to the target client (TC) across three different cases.}
\label{Fig:PFL_heatmap}
\end{figure*}

\noindent 
Before presenting the main theorem, we first establish the following lemmas, with their proofs provided in Appendix~\ref{AppendixC}.
\begin{lemma}\label{lemma1}
    The model parameter update difference of the target client $n$ is calculated as 
    \begin{equation}
    \begin{aligned}
        \Delta\omega_{n}^{(t)} &= \omega_{n}^{(t+1)}-\omega_{n}^{(t)} \\
        &= - \alpha\eta E g_{n}^{(t)} + (1-\alpha)\sum_{m}\pi_{nm}(\omega_{m}^{(t)}-\omega_{n}^{(t)}).
    \end{aligned}
    \end{equation}
\end{lemma}


\begin{lemma}\label{lemma2}
    Under Assumption 4, for $\gamma = \frac{\alpha\eta E}{2} - \frac{1 - \alpha}{2\alpha\eta E}$, $\mathbb{E} \left [ \langle \nabla f_{n}(\omega_{n}^{(t)}, \Delta\omega_{n}^{(t)})\rangle \right ]$ is bounded by

    \begin{equation}\label{eq:26}
        \begin{aligned}
            \mathbb{E} \left [ \langle \nabla f_{n}(\omega_{n}^{(t)}), \Delta\omega_{n}^{(t)}\rangle \right ] &\leq -\gamma \mathbb{E}\left[ ||\nabla f_{n}(\omega_{n}^{t})||^{2}\right]  \\
            & +\frac{\alpha\eta E}{2} \mathbb{E}\left[ ||\nabla f_{n}(\omega_{n}^{t}) - g_{n}^{(t)}||^{2}\right] \\
            &+ \frac{(1-\alpha)\alpha\eta E C}{2}.
        \end{aligned}
    \end{equation}
\end{lemma}


\begin{lemma}\label{lemma3}
    For Eq.~\eqref{eq:26} in Lemma \ref{lemma2}, under Assumptions 1 and 5, $\mathbb{E}\left[ ||\nabla f_{n}(\omega_{n}^{t}) - g_{n}^{(t)}||^{2}\right] $ is bounded as
    \begin{equation}
        \begin{aligned}
            \mathbb{E}\left[ ||\nabla f_{n}(\omega_{n}^{t}) - g_{n}^{(t)}||^{2}\right] \leq L^{2}\eta^{2}E^{2}G^{2}.
        \end{aligned}
    \end{equation}
    
\end{lemma}


\begin{lemma}\label{lemma4}
    According to assumptions 4 and 5, we have:
    \begin{equation}
        \mathbb{E} \left [ ||\Delta\omega_{n}^{(t)}||^{2} \right ] \leq (\alpha\eta EG + (1-\alpha)\sqrt{C})^{2}.
    \end{equation}
\end{lemma}


\begin{theorem}
With the local model update rule defined in Eq.~\eqref{eq:12} for the selected neighbors $m \in [\boldsymbol{M}_{n}]$, the target client's model update as outlined in Eq.~\eqref{eq:2}, and the aggregation of the target client's model according to Eq.~\eqref{eq:5}, and considering Assumptions 1, 2, 4 and 5,  we can show that the pFedWN algorithm achieves $
  \frac{1}{T} \sum_{t=1}^{T} \mathbb{E}[||\nabla f_{n}(\omega_{n}^{(t)})||^{2}] \leq \mathcal{O}(\frac{1}{T})+C',
$  which converges at a rate of $\mathcal{O}(\frac{1}{T})$ for non-convex loss function.
\end{theorem}

\begin{proof}
    The proof is provided in Appendix~\ref{sec:appendix-thm-2}.
\end{proof}

\section{Numerical Evaluation}
\label{sec:results}
In this section, we present numerical results on the performance of the pFedWN algorithm. The implementation setup is separated into a wireless communication model and a learning model setup. Then, we compare our pFedWN algorithm with several FL/PFL baselines.

\subsection{Evaluations Setup}
\textbf{D2D Wireless Communication Setup.} All neighbors and target clients are placed according to the Poisson Point Process (PPP) in the $50 \times 50$ $\text{m}^2$ urban communication area. The neighbor and target client are chosen from distributed clients to form a main link, and other clients act as interferer nodes for the target client. As described, all clients experience a Rayleigh (NLoS) channel in each communication session. Communication model parameters are assigned based on the ISM bands, which are summarized in Table \ref{tab:sim_parameters}. We analyze three different target clients (Case 1, Case 2, and Case 3) in the wireless network as shown in Fig. \ref{Fig:PFL_heatmap}. For example, the Case 1 PFL setting is illustrated in Fig.~\ref{Fig:PFL_heatmap} (a), where the target client (represented by a star) has $10$ neighbors. Based on the placement of the neighbor clients and the channel conditions, the transmission error probabilities are calculated, and the PFL neighbors (nodes with IDs $1, 3, 5$, and $10$) are selected.

\begin{table}[t]
    \centering
    \caption{Communication Model Parameters}
    \label{tab:sim_parameters}
    \begin{tabular}{lc}
        \toprule
        \textbf{Parameter} & \textbf{Value} \\
        \midrule
        Simulation Area & $50 \times 50 \, \text{m}^2$ \\
        Number of Sub-channels ($|\boldsymbol{F}|$) & $14$ \\
        Rayleigh Fading Factor ($\Gamma$) & $2$ \\
        Path Loss Exponent ($\alpha_s$) & $3$ \\
        Reference Distance ($d_0$) & $1$ m \\
        Transmission Power ($\boldsymbol{P}$) & $[0.2]_{1 \times \boldsymbol{S}}$ W \\
        Frequency Operation ($f$) & $2.4$ GHz \\
        Boltzmann Constant ($\kappa$) & $1.38 \times 10^{-23}$ J/K \\
        Noise Temperature ($T$) & $290$ K \\
        Bandwidth ($W$) & $100$ MHz \\
        Channel Fading Threshold ($\boldsymbol{\beta}$) & $[2]_{1 \times \boldsymbol{S}}$ \\
        \bottomrule
    \end{tabular}
    \label{parameters}
\end{table}

\begin{figure*}[t]
    \centering
    \begin{minipage}{.33\linewidth}
    \includegraphics[width=6cm, trim={300 10 230 50}, clip]{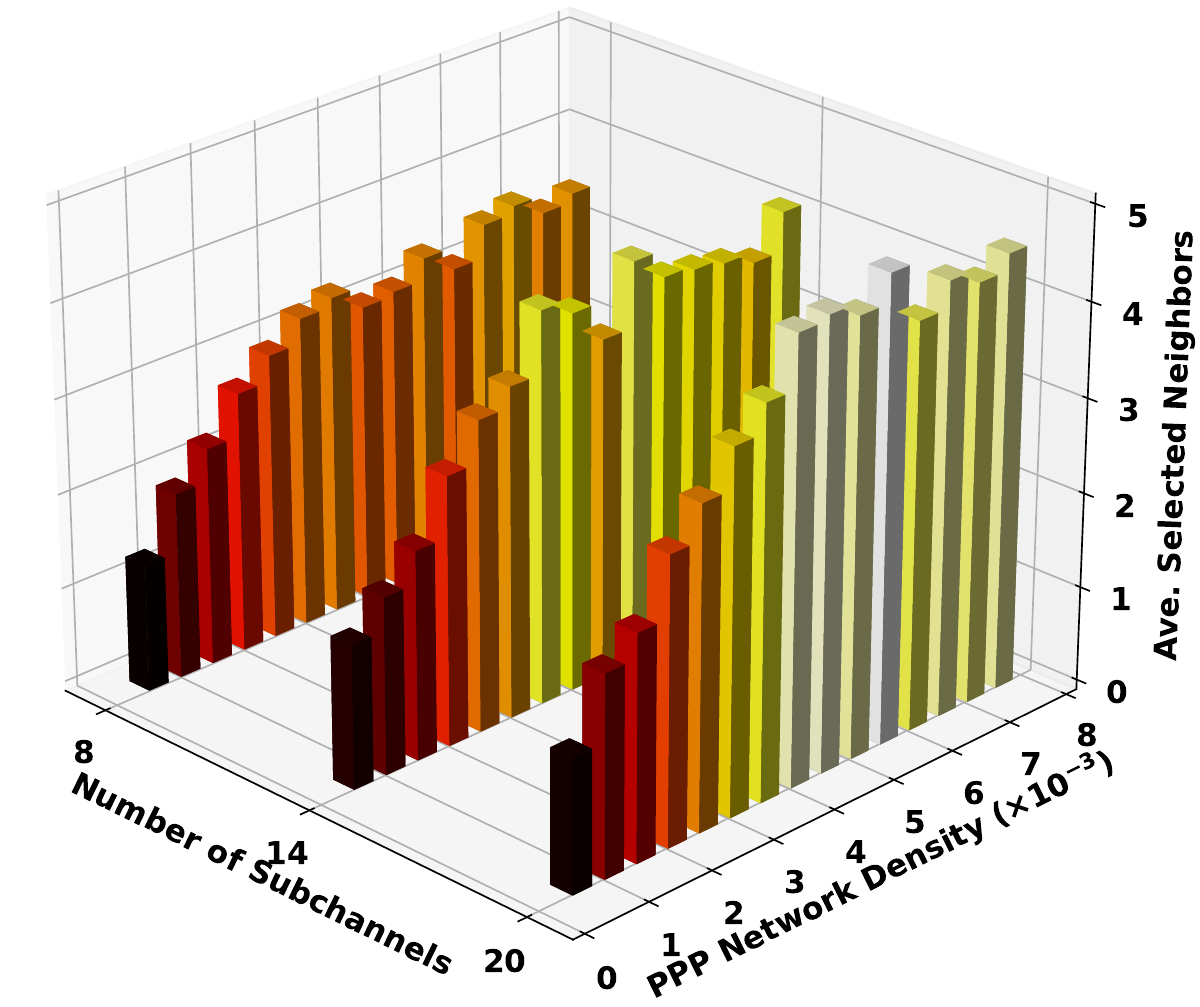}
        \centering (a) SINR Threshold $\gamma_{th}$ = 5
    \end{minipage}%
    \begin{minipage}{.33\linewidth}
   \includegraphics[width=6cm, trim={300 10 230 50}, clip]{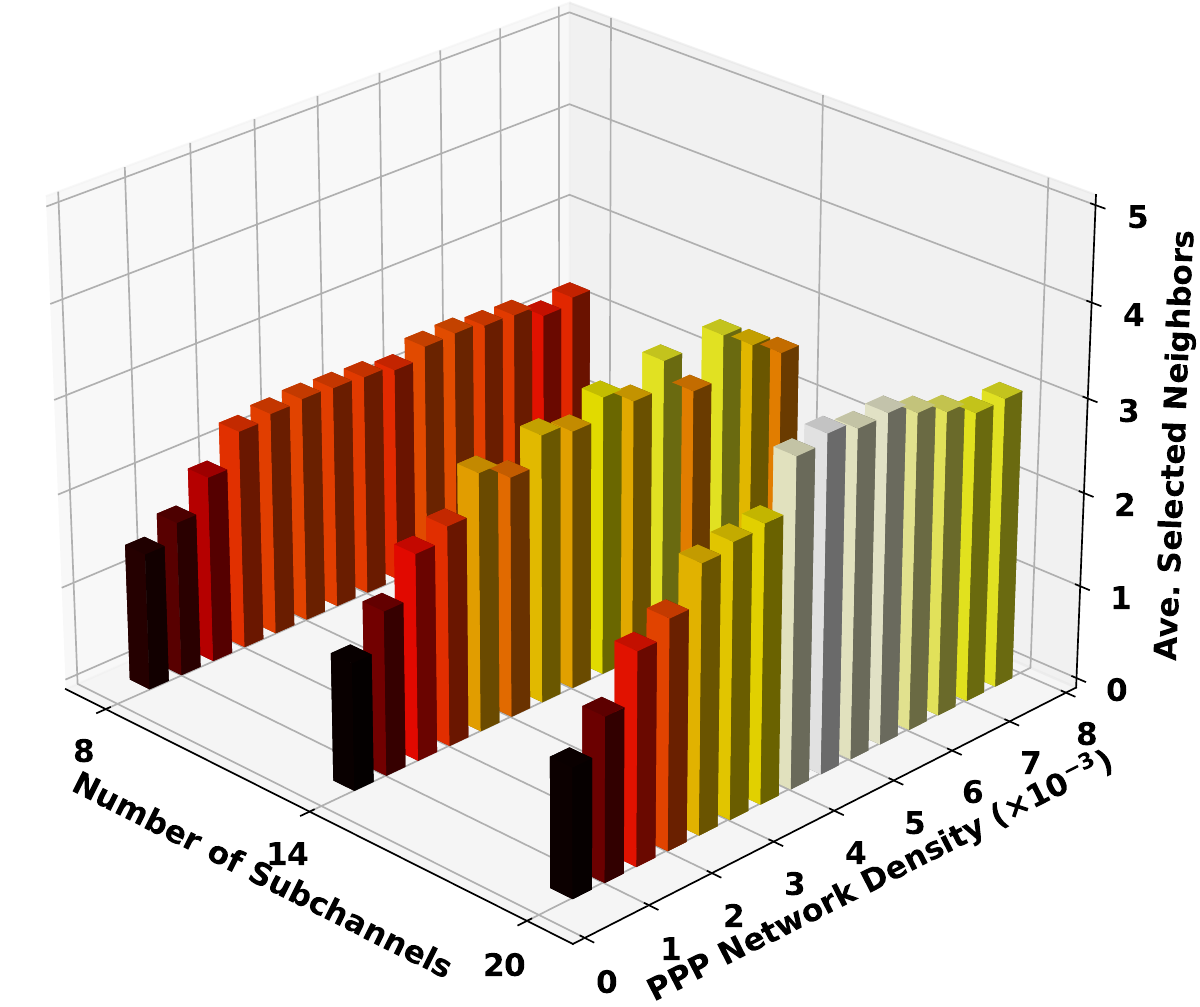}
        \centering (b) SINR Threshold $\gamma_{th}$ = 10
    \end{minipage}%
    \begin{minipage}{.33\linewidth}
    \includegraphics[width=6cm, trim={300 10 230 50}, clip]{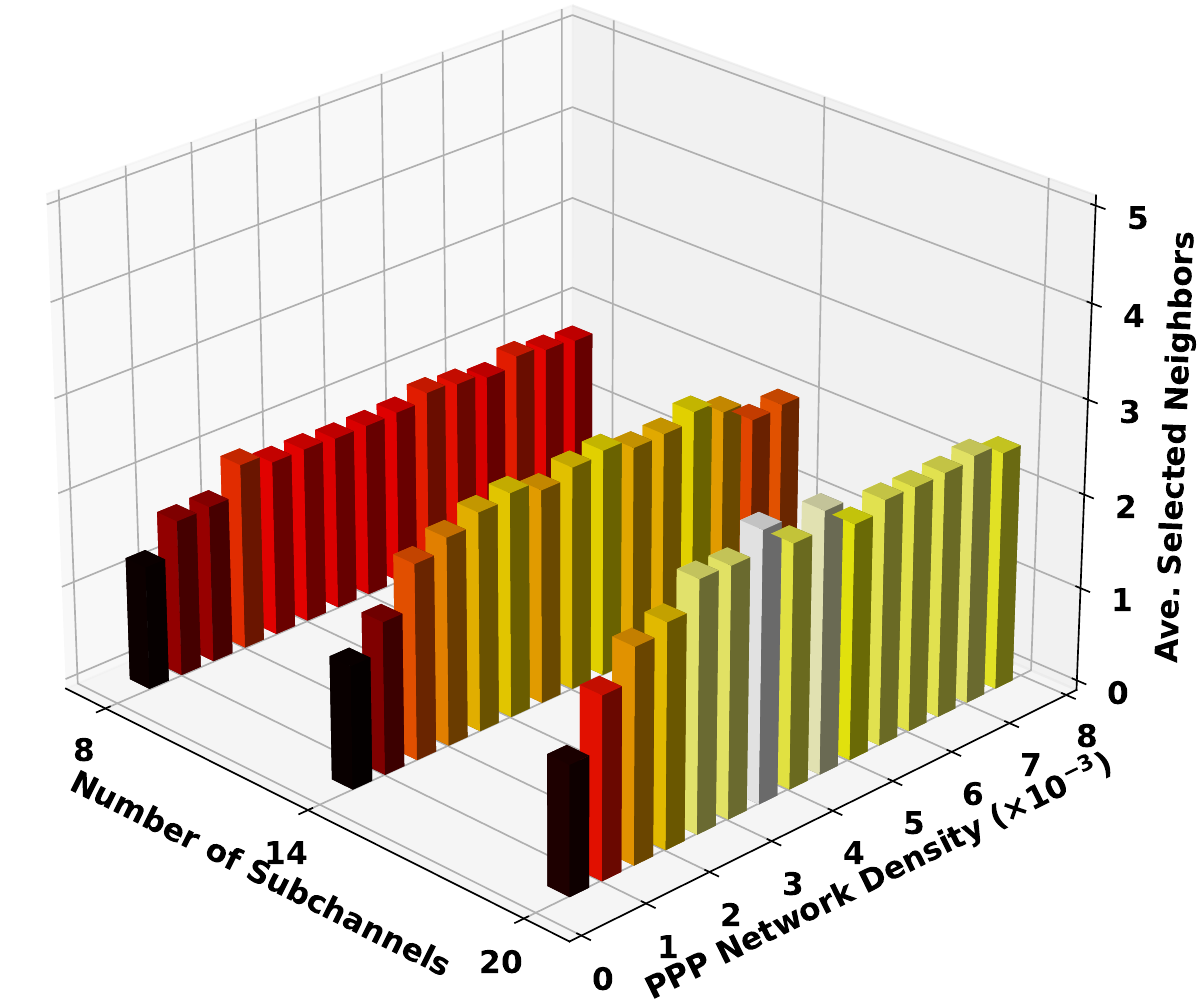}
        \centering (c) SINR Threshold $\gamma_{th}$ = 15
    \end{minipage}%
    \caption{\small Average number of selected neighbors as a function of the number of subchannels $|\boldsymbol{F}|$, SINR thresholds $\gamma_{th}$, and PPP network density.}
\label{Fig:3D_neighbor_selection}
\end{figure*}

\begin{figure}[t]%
    \centering
    \subfloat[\centering $|\boldsymbol{M}_n|$ vs. $|\boldsymbol{G}_n|$ for different $\epsilon$]{{\includegraphics[width=\linewidth, trim={80 5 80 50}, clip]{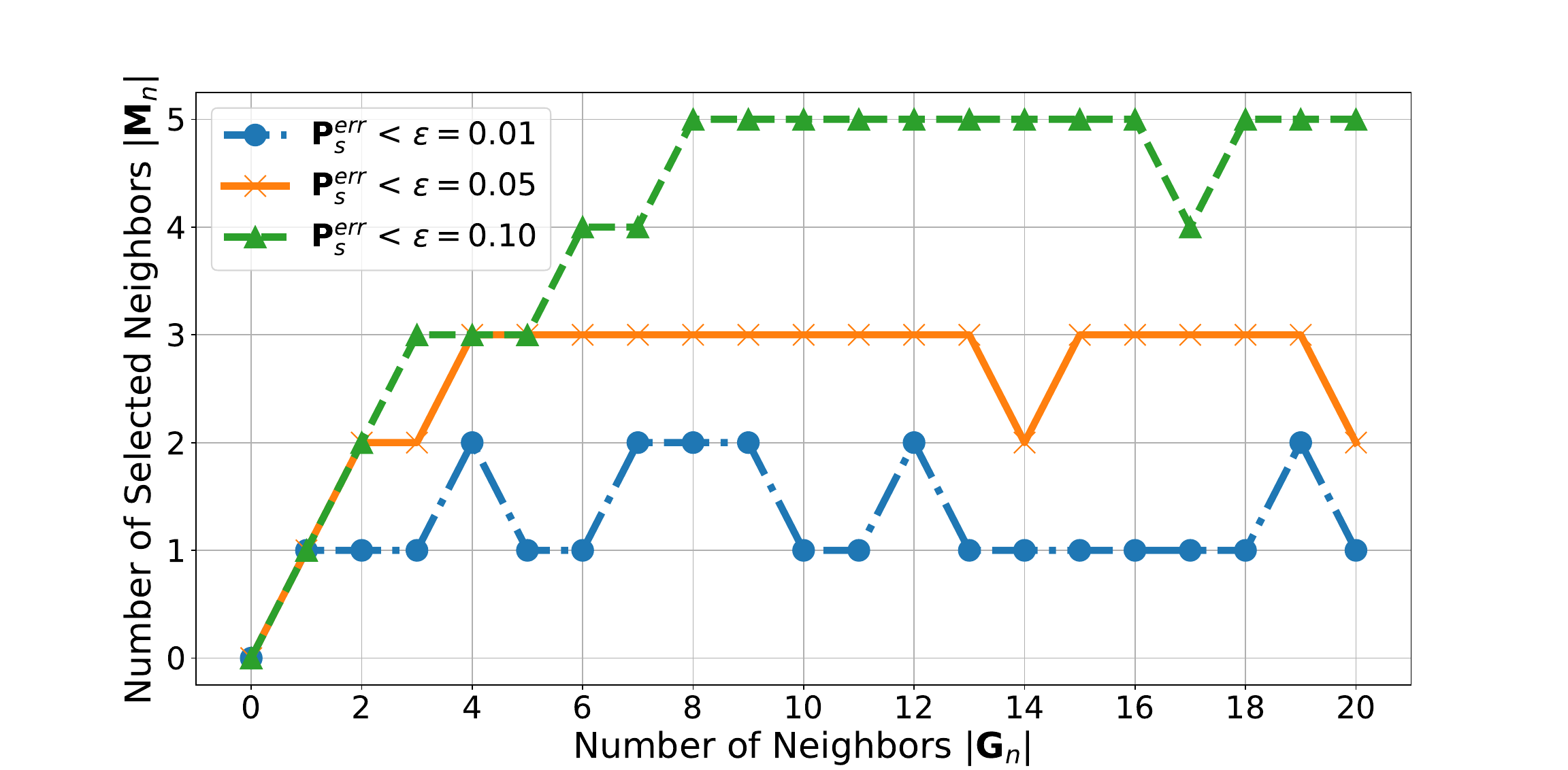} }}%
     \qquad
    \subfloat[\centering $|\boldsymbol{M}_n|$ vs. $|\boldsymbol{G}_n|$ for different SINR thresholds $\gamma_{th}$]{{\includegraphics[width=\linewidth, trim={80 5 80 50}, clip]{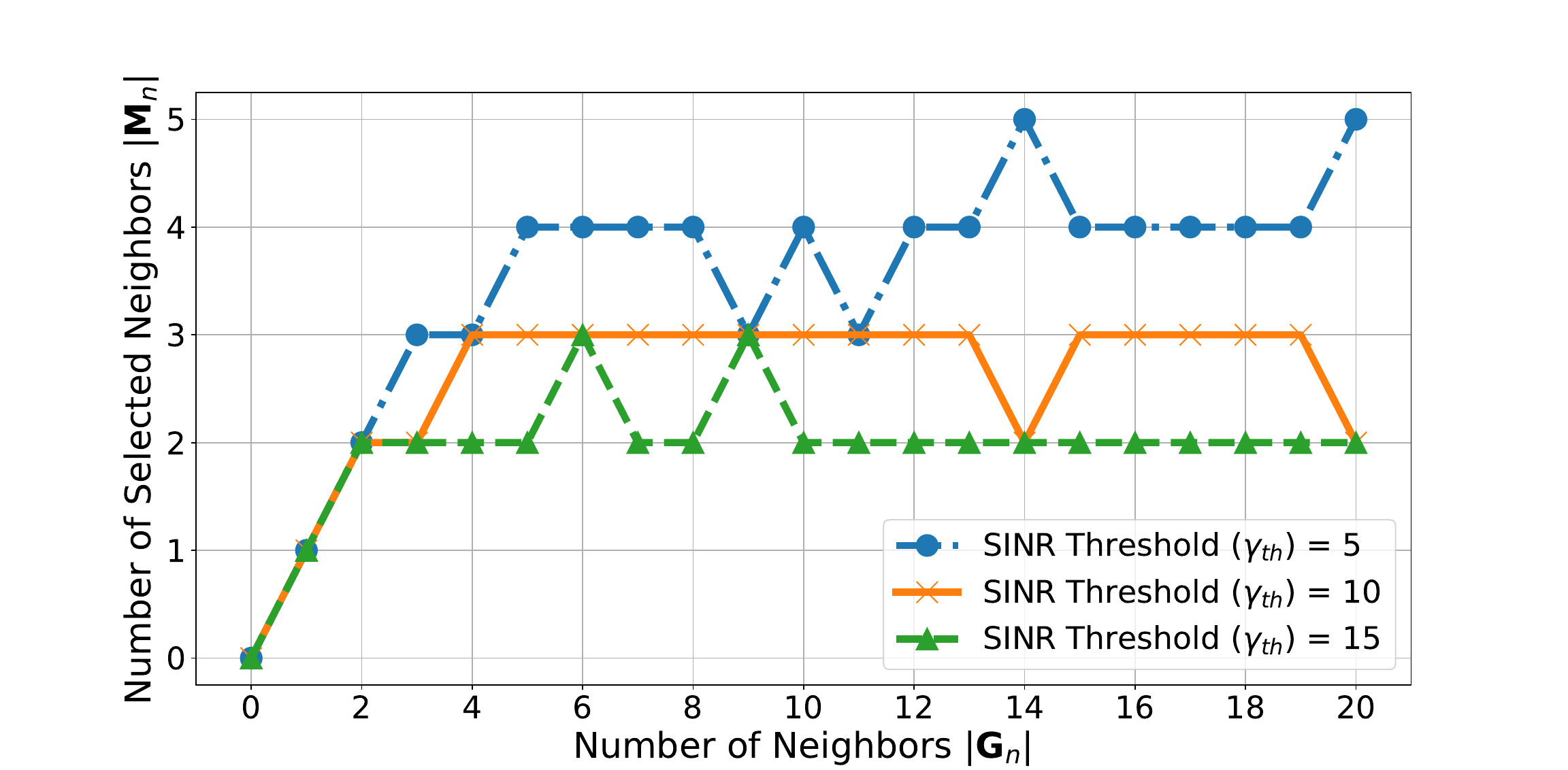} }}
    \caption{\small PFL neighbor selection as the total number of clients within the coverage area increases, evaluated as a function of varying error thresholds ($\epsilon$) and SINR thresholds ($\gamma_{th}$).}%
    \label{Fig:PFL_neighbor_selection}%
\end{figure}

\textbf{D2D PFL Setup.} We evaluate our pFedWN model with different datasets such as image classification tasks with the CIFAR-10, CIFAR-100 \cite{alex2009learning}, and also handwritten character recognition (MNIST) \cite{cohen2017emnist}. 
Since we focus on data heterogeneity, the clients have non-IID datasets following the Dirichlet distribution \cite{lin2020ensemble} with $\alpha_{d} = 0.1$. The classes per client are randomly assigned so that the clients contain a different number of classes and total data samples as their local data in different scenarios, even if data comes from the same dataset (CIFAR-10, CIFAR-100, or MNIST).
For example, each client holds a random number of classes between $1$ to $10$ for the CIFAR-10 dataset. Fig. \ref{Fig:distribution} presents a heatmap of the CIFAR-10 dataset distribution across different clients and three cases. 

\begin{figure}[t]
\centering
\includegraphics[width=\linewidth, trim={20 5 70 5}, clip]{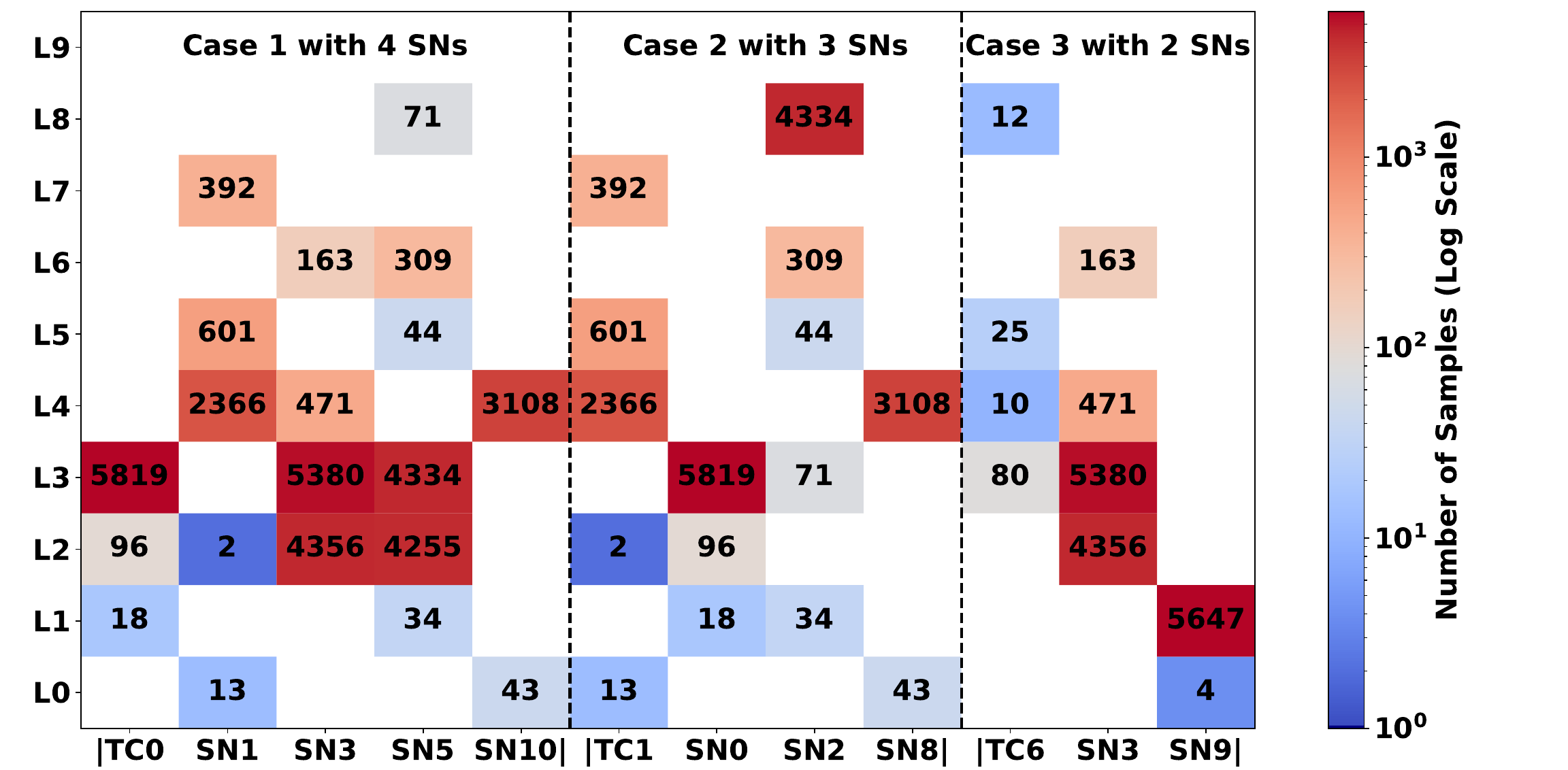}
\caption{\small CIFAR-10 data distribution heatmap for 3 cases with different SNs as shown in Fig. \ref{Fig:PFL_heatmap} (L0 to L9: different labels).}
\label{Fig:distribution}
\end{figure}

For the neural networks, we implement a 3-layer CNN for MNIST and ResNet18 for both CIFAR-10 and CIFAR-100 datasets~\cite{zhang2023pfllib}. For each client, we randomly split their local data with a ratio of $75\%-25\%$ train-test. In our training, we have $100$ communication rounds, and the local model parameters are updated with $1$ epoch of local training. For all baseline methods, only neighbors selected based on the communication selection method join the training process. 

\textbf{Baseline methods.} We compare the pFedWN with traditional FL algorithms such as FedAvg \cite{mcmahan2017communication} and FedProx \cite{wang2020tackling}. In addition, we compare pFedWN with state-of-art PFL approaches, such as Per-FedAvg \cite{fallah2020personalized} and FedAMP \cite{huang2021personalized}, as well as a local training baseline. For the target client, we train the model under both 10 and 20-neighbor (defined by the density of clients in the network) scenarios.  



\subsection{Impacts of Channel Conditions on PFL Neighbor Selection} Fig.~\ref{Fig:PFL_neighbor_selection} (a) shows the number of selected neighbors $|\boldsymbol{M}_n|$ from all neighbors $|\boldsymbol{G}_n|$ for a specific target client. In this figure, the number of selected neighbors is rounded by taking the average of 20 iterations for each number of neighbors.
From the results, we note that as the threshold $\epsilon$ increases, the number of selected neighbors increases, as expected. To select the desired neighbors, the path loss is the dominant factor when the number of available neighbors $|\boldsymbol{G}_n|$ is small. However,  as the number of neighbors increases (more than 10 neighbors), the interferer neighbors will become dominant in the area. Therefore, the number of selected neighbors will not increase even though more neighbors are available.


In Fig.~\ref{Fig:PFL_neighbor_selection} (b), the number of selected neighbors $|\boldsymbol{M}_n|$ is shown as the number of neighbors $|\boldsymbol{G}_n|$ increases, and for various SINR thresholds $\gamma_{th}$ and error threshold $\epsilon = 0.05$.  
From the results, we notice that as $\gamma_{th}$ increases, the average number of selected neighbors decreases because the data packets cannot be decoded correctly at the target client, and $\mathbf{P}_{s}^{err}$ increases. 

\begin{table*}[t]
\centering
\caption{\small{Maximum test accuracy (percentage) for different \textbf{target client cases} in a 10-neighbor network (best results in bold).}}
\label{tab:10_acc}
\setlength{\tabcolsep}{6pt}
\renewcommand{\arraystretch}{1.3}
\begin{tabular}{@{}lccccccccc@{}}
\toprule
\multicolumn{1}{c}{\textbf{}} & \multicolumn{3}{c}{Case 1 with 4 SNs ($\gamma_{th} = 5$)} & \multicolumn{3}{c}{Case 2 with 3 SNs ($\gamma_{th} = 10$)} & \multicolumn{3}{c}{Case 3 with 2 SNs ($\gamma_{th} = 15$)} \\ \cmidrule(lr){2-4} \cmidrule(lr){5-7} \cmidrule(lr){8-10}
Method & MNIST & CIFAR-10 & CIFAR-100 & MNIST & CIFAR-10 & CIFAR-100 & MNIST & CIFAR-10 & CIFAR-100 \\
\midrule
Local       & 98.3 & 98.3 & 55.8 & 98.4 & 65.6 & 52.3 & 98.0 & 61.9 & 45.4 \\
FedAvg      & 88.0 & 16.0 & 35.0 & 96.2 & 31.3 & 35.4 & 88.5 & 54.8 & 30.2 \\
FedProx     & 83.4 & 19.2 & 32.8 & 94.3 & 31.2 & 36.1 & 88.1 & 56.2 & 32.3 \\
Per-FedAvg  & 92.0 & 95.0 & 38.2 & 98.1 & 40.6 & 39.7 & 90.8 & 64.4 & 32.2 \\
FedAMP      & 98.4 & 98.3 & 56.6 & 98.3 & 62.3 & 51.6 & 98.1 & 64.9 & 48.5 \\
pFedWN     & \textbf{98.4} & \textbf{98.4} & \textbf{57.2} & \textbf{98.4} & \textbf{68.8} & \textbf{58.8} & \textbf{98.1} & \textbf{65.3} & \textbf{50.2} \\
\bottomrule
\end{tabular}
\end{table*}

Fig. \ref{Fig:3D_neighbor_selection} demonstrates 3D plots for different SINR thresholds $\gamma_{th}$, while the numbers of subchannels $|\boldsymbol{F}|$ and PPP network density varies in each subfigure. Notably, for a PPP network density ranging from $0.5 \times 10^{-3}$ to $7.5 \times 10^{-3}$ within a $50 \times 50 \, \text{m}^2$ communication area, the number of nodes approximately rises from $2$ to $30$. The results show the average number of selected neighbors with error threshold $\epsilon = 0.05$ over $100$ iterations. From the results, we observe that when the number of subchannels $|\boldsymbol{F}|$ increases from $8$ to $20$, the average number of selected neighbors increases due to smaller amounts of interference. Furthermore, by increasing the SINR threshold $\gamma_{th}$ from $5$ (Fig. \ref{Fig:3D_neighbor_selection} (a)) to $15$ (Fig. \ref{Fig:3D_neighbor_selection} (c)), the average number of selected neighbors decreases.
Finally, we note that increasing the PPP network density does not necessarily increase the average number of selected neighbors due to increased interference.

\subsection{Effectiveness of the EM Algorithm and PFL Results}
We implemented the EM algorithm to evaluate its performance in the CIFAR-10 scenario.
Specifically, we use the target client in Case 1 as an example (for the sake of brevity and space limitations), and each of its selected neighbors contained a non-IID data distribution that has a different number of categories and total data samples, which can be found in Fig.  \ref{Fig:distribution}. As illustrated in Fig. \ref{Fig:weights}, the EM algorithm effectively assigned more weights to neighbor 5, which has similar data distributions, but less to neighbor 10, which has a totally different data distribution compared with the target client in Case 1. Therefore, our results demonstrate that the EM algorithm can effectively capture distributional similarities for CIFAR-10 datasets with non-IID and unbalanced datasets. 

\begin{figure}[t]%
    \centering {{\includegraphics[width=\linewidth, trim={40 0 110 50}, clip]{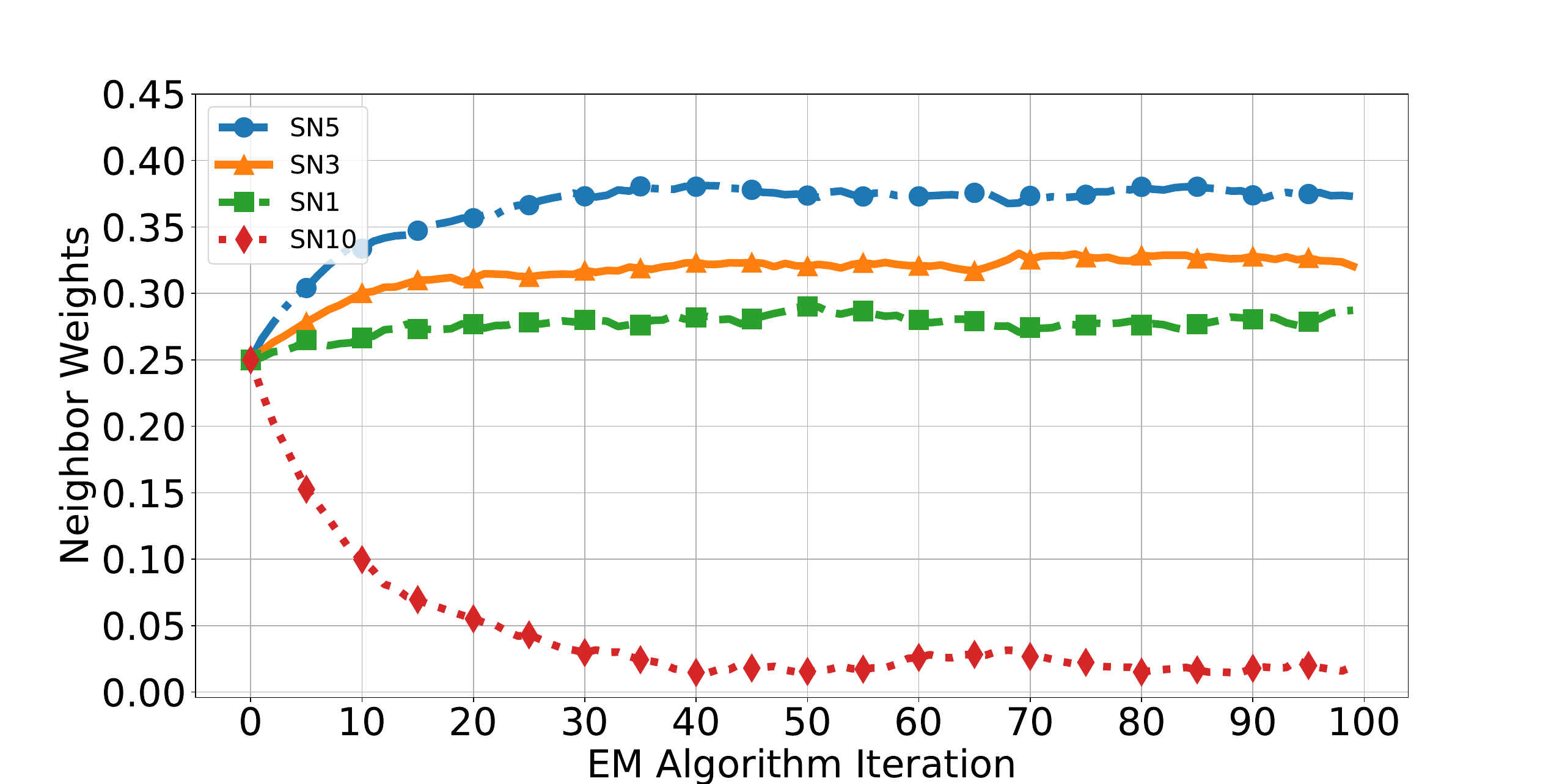} }}%
    \caption{\small Weight convergence of the EM algorithm with CIFAR-10 dataset for selected neighbors of target client in Case 1.}%
    \label{Fig:weights}%
\end{figure}

%
To evaluate the performance of the pFedWN algorithm, next, we present the learning performance of individual target clients
for two network scenarios with $10$  and $20$ neighbors.  

\subsubsection{10-neighbor Network Performance}
For the 10-neighbor network setup, 
Fig. \ref{Fig:PFL_heatmap}  demonstrates a distribution of all nodes in a given area, including the neighbors and selected target clients (shown as the star) for three different cases. 
Furthermore, the heatmap illustrates the probability of transmission error of each neighbor linked to the target client. According to Fig. \ref{Fig:PFL_heatmap} (a), $4$ out of $10$ neighbors are selected with SINR threshold $\gamma_{th} = 5$ from the target client's coverage area in the 10-neighbor network that participates in pFedWN training for Case 1. Additionally, $3$ and $2$ neighbors are selected for the target client in Cases 2 and 3 with SINR threshold $\gamma_{th} = 10$ and $15$, respectively. For other FL baseline methods, the same clients in the coverage area join the learning process for those methods. 

The learning performance of these three cases is reported in Table~\ref{tab:10_acc}. We observe that pFedWN outperforms other baseline methods. As previously mentioned, our focus is on the learning performance of the target client. Therefore, with various datasets, both the data size and the distributions differ for these target clients. According to Table \ref{tab:10_acc}, it is also evident that the target clients' learning performance, across all datasets, is improved by collaborating with selected neighbors compared to its local learning alone.

\begin{figure*}[t]
    \centering
    \begin{minipage}{.5\linewidth}
        \includegraphics[width=9cm, trim={50 5 90 60}, clip]{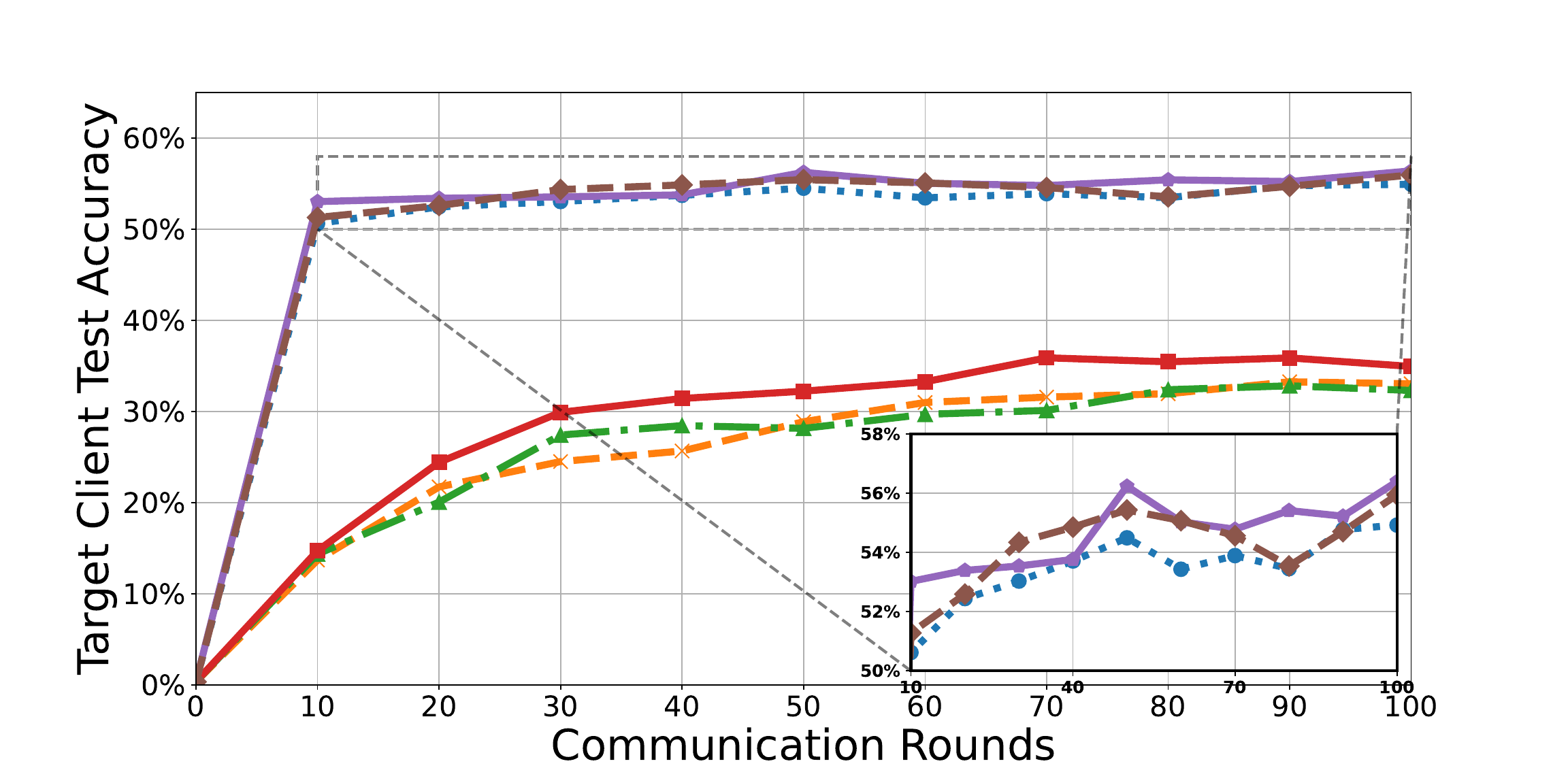}
        \centering (a) 10-neighbor network.
    \end{minipage}%
    \begin{minipage}{.5\linewidth}
        \includegraphics[width=9cm, trim={50 5 90 60}, clip]{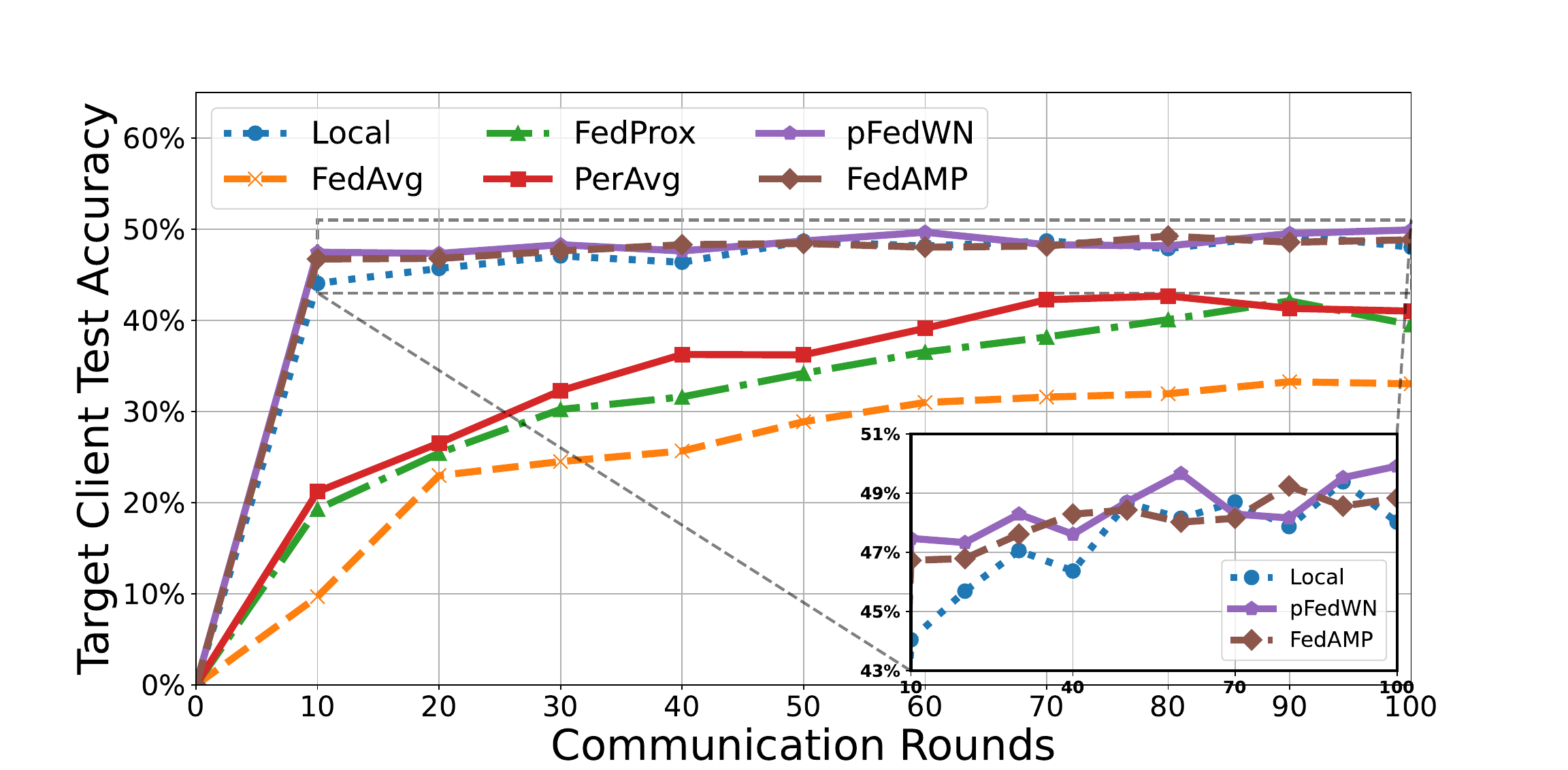}
        \centering (b) 20-neighbor network.
    \end{minipage}%
    \caption{\small 10-neighbor vs. 20-neighbor networks performance on CIFAR-100 dataset.}
    \label{Fig:subacc}
\end{figure*}

Furthermore, traditional FL strategies, such as FedAvg and FedProx, show significant underperformance due to unbalanced and non-IID data distributions across clients. In scenarios where the target client already has a large data set and can achieve satisfactory learning performance through local learning, forcing this client to participate in FL may not be advantageous. For instance, Fig.~\ref{Fig:distribution} illustrates the data distributions of each client in the system for each case using the CIFAR-10 dataset. In Case 1, the target client primarily contains labels 1, 2, and 3, with a majority of samples under label 3 ($5,819$ out of $5,933$ total samples). In contrast, other clients, like neighbor 10 with $3,151$ samples, show significantly different distributions, mainly consisting of label 4, with no data for labels 1, 2, and 3. This discrepancy severely impacts the performance of FedAvg, where the accuracy of the target client dramatically falls to $16.0\%$ when testing the global model, compared to a local accuracy of $98.3\%$.  While FedProx attempts to mitigate this issue through regularization, it achieves a slight improvement with an accuracy of $19.2\%$. While Per-FedAvg introduces some level of personalization for the target client, its performance still falls short of that achieved through local learning. On the other hand, we observe that pFedWN consistently achieves better performance with unbalanced and non-IID data distributions.





\begin{table}[t]
\centering
\caption{\small{Maximum test accuracy in percentage for \textbf{target client} in a 20-neighbor network with 3 SNs ($\gamma_{th} = 10$) (best results in bold).}}
\label{tab:20_acc}
\setlength{\tabcolsep}{6pt}
\renewcommand{\arraystretch}{1.3}
\begin{tabular}{@{}lccc@{}}
\toprule
Method & MNIST & CIFAR-10 & CIFAR-100 \\
\midrule
Local       & 83.3   & 53.2    & 49.7    \\
FedAvg      & 86.1   & 37.9    & 42.1    \\
FedProx     & 88.9   & 32.5    & 42.2    \\
Per-FedAvg  & 88.9   & 48.3    & 42.7    \\
FedAMP      & \textbf{92.1}  & \textbf{58.1}    & 49.6    \\
pFedWN     & 90.0   & 56.5    & \textbf{50.8} \\
\bottomrule
\end{tabular}
\end{table}

\subsubsection{{20-neighbor Network Performance}}
For this scenario, we implement our pFedWN algorithm in a relatively larger network where a target client has $20$ neighbors randomly distributed within the target client's coverage area. Based on our communication setting with SINR threshold $\gamma_{th} = 10$, only $3$ clients out of these $20$ neighbors are selected as the PFL neighbors for the target client. We use one target client's learning performance as an example, shown in Table \ref{tab:20_acc}. pFedWN outperforms other baseline learning methods across all three datasets except FedAMP for MNIST and CIFAR-10 datasets. This is because the FedAMP tries to personalize models to each client, but it still operates with a form of global consensus at its core, which can help in scenarios where including more clients provides a richer pool of data especially beneficial for simpler or relatively lower-completely tasks such as MNIST and CIFAR10. However, with the increase in data heterogeneity and task complexity in a larger network, our algorithm performs better than FedAMP because the target client can more aggressively filter out unhelpful neighbors via connectivity-based weighting. Moreover, compared to the 10-neighbor network scenario, there is a significant decrease in local learning performance. This decline is due to the distribution of the same quantity of data samples (60,000) among a larger number of clients in the network, resulting in a reduction of local data samples available to each client, which in turn decreases the testing accuracy. It is also noteworthy that  FL algorithm \emph{may}  exhibit improved performance compared to local learning in this context (e.g., MNIST dataset). This suggests collaboration with others to improve learning performance could be beneficial, especially when the available dataset is relatively small.

Finally, Fig.~\ref{Fig:subacc} compares the learning performance of the 10-neighbor vs. 20-neighbor network with the CIFAR-100 dataset.  
Clients in the 10-neighbor network exhibit enhanced learning stability and achieve high test accuracy since the larger individual datasets are available due to fewer network participants. With a total of 60,000 samples, a small FL network enables clients to achieve higher accuracy through local learning.
Conversely, for the 20-neighbor network, while individual local performance declines due to smaller datasets, the overall FL performance improves as the model benefits from broader collaboration. Furthermore, the pFedWN approach improves the learning performance for both networks.

\section{Conclusion}
\label{sec:concludes}
In this work, we introduce a PFL approach for D2D wireless networks. Our problem formulation incorporates wireless channel conditions as a criterion for PFL neighbor selection. Using the EM algorithm, our proposed pFedWN algorithm captures data similarity for optimal weight assignment to the selected neighbors during the learning process. Furthermore, the proposed pFedWN enables the target client to adjust the model update between its local and aggregated models from its PFL neighbors. Numerical evaluations validate that the pFedWN approach improves the learning performance of the target client but also reduces communication overhead by limiting the number of PFL neighbors selected. In the future, we plan to enhance the pFedWN approach by incorporating additional dimensions of client heterogeneity, including variations in computing power and desired performance metrics, thereby investigating multi-objective personalized federated learning over wireless networks. 

\section*{Acknowledgment}
The material is based upon work supported by NSF grants 1955561, 2212565, and 2323189. Any opinions, findings, and conclusions or recommendations expressed in this material are those of the author(s) and do not necessarily reflect the views of the NSF.

\appendices
\section{Parameters of Log-normal Distribution}
\label{log-normaldistribution}
As previously discussed, the PDF of interference model $I_s^f(\hat{\boldsymbol{h}}^f_{-s})$, which follows Log-normal distribution, can be described as \cite{Tian-2016-Interference}:
\begin{equation}
    \mathbf{P}(I_s^f(\hat{\boldsymbol{h}}^f_{-s}) = x) = \frac{\exp(-\frac{(\ln x-\mu(\hat{\boldsymbol{h}}^f_{-s}))^2}{2\sigma^2(\hat{\boldsymbol{h}}^f_{-s})})}{x\sigma(\hat{\boldsymbol{h}}^f_{-s})\sqrt{2\pi}},
\end{equation}
where $\mu(\hat{\boldsymbol{h}}^f_{-s})$ and $\sigma(\hat{\boldsymbol{h}}^f_{-s})$ are the location and scale parameters, are defined as:
\begin{align*}
\mu(\hat{\boldsymbol{h}}^f_{-s}) = \ln(\Tilde{E}[I_s^f(\hat{\boldsymbol{h}}^f_{-s})]) - \frac{1}{2} \ln(1 + \frac{\Tilde{D}[I_s^f(\hat{\boldsymbol{h}}^f_{-s})]}{(\Tilde{E}[I_s^f(\hat{\boldsymbol{h}}^f_{-s})])^2}),
\end{align*}
and,
\begin{align*}
\sigma(\hat{\boldsymbol{h}}^f_{-s}) = \sqrt{\ln(1 + \frac{\Tilde{D}[I_s^f(\hat{\boldsymbol{h}}^f_{-s})]}{(\Tilde{E}[I_s^f(\hat{\boldsymbol{h}}^f_{-s})])^2})}.
\end{align*}
Here, $\Tilde{E}[I_s^f(\hat{\boldsymbol{h}}^f_{-s})]$ and $\Tilde{D}[I_s^f(\hat{\boldsymbol{h}}^f_{-s})]$ are the first and second order moments of $I_s^f(\hat{\boldsymbol{h}}^f_{-s})$, which are expressed as:
\begin{align*}
& \Tilde{E}[I_s^f(\hat{\boldsymbol{h}}^f_{-s})] = \sum_{r\in \boldsymbol{S} \backslash s}P_r \mathbb{E}\bigl[(\hat{h}_{r}^f)^2(\Tilde{h}_{r}^f)^2\alpha_r^f(\beta_r)\bigl] = \\ &
\sum_{r\in \boldsymbol{S} \backslash s} P_r(\hat{h}_{r}^f)^2\int_{\beta_r}^\infty\frac{2x^3}{\Gamma}e^{-\frac{x^2}{\Gamma}}dx \frac{1}{|\boldsymbol{F}|}\bigl(1-(1-e^{-\frac{\beta_r^2}{\Gamma}})^{|\boldsymbol{F}|}\bigl),
\end{align*}
and, 
\begin{align*}
& \Tilde{D}[I_s^f(\hat{\boldsymbol{h}}^f_{-s})] = \sum_{r\in \boldsymbol{S} \backslash s} P_r^2(\hat{h}_{r}^f)^4 \int_{\beta_r}^\infty \frac{2x^5}{\Gamma} e^{-\frac{x^2}{\Gamma}}dx \times \\ & \frac{1}{|\boldsymbol{F}|^2}\bigl(1-(1-e^{-\frac{\beta_r^2}{\Gamma}})^{|\boldsymbol{F}|}\bigl)^2 + \sum_{r_1 \neq r_2 \in \boldsymbol{S} \backslash s} P_{r_1}(\hat{h}_{r_1}^f)^2 \times \\ & \int_{\beta_{r_1}}^\infty  \frac{2x^3}{\Gamma}e^{-\frac{x^2}{\Gamma}}dx \frac{1}{|\boldsymbol{F}|}\bigl(1-(1-e^{-\frac{\beta_{r_1}^2}{\Gamma}})^{|\boldsymbol{F}|}\bigl) P_{r_2}(\hat{h}_{r_2}^f)^2 \times \\ & \int_{\beta_{r_2}}^\infty \frac{2x^3}{\Gamma}e^{-\frac{x^2}{\Gamma}}dx \frac{1}{|\boldsymbol{F}|} \bigl(1-(1-e^{-\frac{\beta_{r_2}^2}{\Gamma}})^{|\boldsymbol{F}|}\bigl)-(\Tilde{E}[I_s^f(\hat{\boldsymbol{h}}^f_{-s})])^2.
\end{align*}

\section{Derivations of Expectation-Maximization Algorithm in pFedWN}
\noindent
To find the optimal weights of the selected neighbors for a specific target client $n$, we employ the EM algorithm. We follow the derivation and proof provided in \cite{marfoq2021federated} and adapt it to our setting. The objective is to learn parameters \(\{\omega_{nm}\}\) and weights \(\{\pi_{nm}\}\) from the data $D_n$ of the target client by maximizing the likelihood \( p(D_n | \Omega, \Pi) \). We introduce functions \(q_n(z)\) such that \(q_n \geq 0\) and \(\sum_{m=1}^{M_n} q_n(z = m) = 1\) in the expression of the likelihood. For \(\Omega \in \mathbb{R}^{M_n \times d}\) and \(\Pi \in \Delta^{M_n}\), we have:
\begin{align*}
      & \log p(D_n | \Omega, \Pi) = \sum_{i=1}^{k_n} \log p_n(\mathbf{x}_i^{(n)}, y_i^{(n)} | \Omega, \pi_n) \\ &
        = \sum_{i=1}^{k_n} \log \left[ \sum_{m=1}^{M_n} \left( p_n(\mathbf{x}_i^{(n)}, y_i^{(n)}, z_n^{(i)} = m | \Omega, \pi_n) \cdot q_n(z_n^{(i)} = m) \right) \right] \\ &
        \geq \sum_{i=1}^{k_n} \sum_{m=1}^{M_n} q_n(z_n^{(i)} = m) \log \left( \frac{p_n(\mathbf{x}_i^{(n)}, y_i^{(n)}, z_n^{(i)} = m | \Omega, \pi_n)}{q_n(z_n^{(i)} = m)} \right) \\ &
        = \sum_{i=1}^{k_n} \sum_{m=1}^{M_n} q_n(z_n^{(i)} = m) \bigl( \log p_n(\mathbf{x}_i^{(n)}, y_i^{(n)}, z_n^{(i)} = m | \Omega, \pi_n) \\
        &\quad - \log q_n(z_n^{(i)} = m) \bigr) \\ &
        = \sum_{i=1}^{k_n} \sum_{m=1}^{M_n} q_n(z_n^{(i)} = m) \log p_n(\mathbf{x}_i^{(n)}, y_i^{(n)}, z_n^{(i)} = m | \Omega, \pi_n)\\ &\quad - \sum_{i=1}^{k_n} \sum_{m=1}^{M_n} q_n(z_n^{(i)} = m) \log q_n(z_n^{(i)} = m)
        \triangleq \mathcal{L}(\Omega, \Pi, Q_n),
\end{align*}

\noindent
where we used Jensen's inequality because \(\log\) is concave. \(\mathcal{L}(\Omega, \Pi, Q_n)\) is an evidence lower bound. The centralized EM algorithm corresponds to iteratively maximizing this bound with respect to \(Q_n\) (E-STEP) and with respect to \(\{\Omega, \Pi\}\) (M-STEP).

\noindent 
\textbf{(1) E-STEP.}
The difference between the log-likelihood and the evidence lower bound \(\mathcal{L}(\Omega, \Pi, Q_n)\) can be expressed in terms of the sum of KL divergences:
\begin{align*}
        &\log p(D_n | \Omega, \Pi) - \mathcal{L}(\Omega, \Pi, Q_n) \nonumber \\
        &= \sum_{i=1}^{k_n} \bigl\{ \log p_n(\mathbf{x}_i^{(n)}, y_i^{(n)} | \Omega, \pi_n) \nonumber \\
        &\quad - \sum_{m=1}^{M_n} q_n(z_n^{(i)} = m) \log \frac{p_n(\mathbf{x}_i^{(n)}, y_i^{(n)}, z_n^{(i)} = m | \Omega, \pi_n)}{q_n(z_n^{(i)} = m)} \bigr\} \nonumber \\
        &= \sum_{i=1}^{k_n} \sum_{m=1}^{M_n} q_n(z_n^{(i)} = m) \bigl( \log p_n(\mathbf{x}_i^{(n)}, y_i^{(n)} | \Omega, \pi_n) \nonumber \\
        &\quad - \log \frac{p_n(\mathbf{x}_i^{(n)}, y_i^{(n)}, z_n^{(i)} = m | \Omega, \pi_n)}{q_n(z_n^{(i)} = m)} \bigr) \nonumber \\
        &= \sum_{i=1}^{k_n} \sum_{m=1}^{M_n} q_n(z_n^{(i)} = m) \bigl( \log p_n(\mathbf{x}_i^{(n)}, y_i^{(n)} | \Omega, \pi_n) \nonumber \\
        &\quad - \log p_n(\mathbf{x}_i^{(n)}, y_i^{(n)}, z_n^{(i)} = m | \Omega, \pi_n) + \log q_n(z_n^{(i)} = m) \bigr) \nonumber \\
        &= \sum_{i=1}^{k_n} \sum_{m=1}^{M_n} q_n(z_n^{(i)} = m) \log \frac{q_n(z_n^{(i)} = m)}{p_n(z_n^{(i)} = m | \mathbf{x}_i^{(n)}, y_i^{(n)}, \Omega, \pi_n)} \nonumber \\
        &= \sum_{i=1}^{k_n} \text{KL} \left( q_n(z_n^{(i)}) \| p_n(z_n^{(i)} | \mathbf{x}_i^{(n)}, y_i^{(n)}, \Omega, \pi_n) \right) \geq 0.
\end{align*}

\noindent
For fixed parameters \(\{\Omega, \Pi\}\), the maximum of \(\mathcal{L}(\Omega, \Pi, Q_n)\) is reached when
\begin{align*}
    \sum_{i=1}^{k_n} \text{KL} \left( q_n(z_n^{(i)}) \| p_n(z_n^{(i)} | \mathbf{x}_i^{(n)}, y_i^{(n)}, \Omega, \pi_n) \right) = 0.
\end{align*}

\noindent
Thus for \( i \in [k_n] \), we have:
\begin{align*}
        &q_n(z_n^{(i)} = m) = p_n(z_n^{(i)} = m | \mathbf{x}_i^{(n)}, y_i^{(n)}, \Omega, \pi_n) \\
        &= \frac{p_n(\mathbf{x}_i^{(n)}, y_i^{(n)} | z_n^{(i)} = m, \Omega, \pi_n) \cdot p_n(z_n^{(i)} = m | \Omega, \pi_n)}{p_n(\mathbf{x}_i^{(n)}, y_i^{(n)} | \Omega, \pi_n)} \\
        &= \frac{p_m(\mathbf{x}_i^{(n)}, y_i^{(n)} | \omega_m) \cdot \pi_{nm}}{\sum_{m'=1}^{M_n} p_{m'}(\mathbf{x}_i^{(n)}, y_i^{(n)} | \omega_{m'}) \cdot \pi_{nm'}} \\
        &=  \frac{p_m(y_i^{(n)} | \mathbf{x}_i^{(n)}, \omega_m) \cdot \pi_{nm}}{\sum_{m'=1}^{M_n} p_{m'}(y_i^{(n)} | \mathbf{x}_i^{(n)}, \omega_{m'}) \cdot \pi_{nm'}}.
\end{align*}

\noindent 
\textbf{(2) M-STEP.}
Now, we maximize \(\mathcal{L}(\Omega, \Pi, Q_n)\) with respect to \(\{\Omega, \Pi\}\). By dropping the terms not depending on \(\{\Omega, \Pi\}\) in the expression of \(\mathcal{L}(\Omega, \Pi, Q_n)\) we write:
\begin{align*}
&\mathcal{L}(\Omega, \Pi, Q_n) \nonumber \\
&= \sum_{i=1}^{k_n} \sum_{m=1}^{M_n} q_n(z_n^{(i)} = m) \bigl[ \log p_{\omega_m}(\mathbf{x}_i^{(n)}, y_i^{(n)}) + \log \pi_{nm} \bigr] + c \nonumber \\
&= \sum_{i=1}^{k_n} \sum_{m=1}^{M_n} q_n(z_n^{(i)} = m) \bigl[ \log p_{\omega_m}(y_i^{(n)} | \mathbf{x}_i^{(n)}) + \log p_m(\mathbf{x}_i^{(n)}) \nonumber\\
&\quad  + \log \pi_{nm} \bigr] + c, 
\end{align*}
where \(c\) is a constant not depending on \(\{\Omega, \Pi\}\).
Thus, for \( m \in [\boldsymbol{M}_n] \), by solving a simple optimization problem we update \(\pi_{nm}\) as follows: 
\begin{align*}
    \pi_{nm} = \frac{\sum_{i=1}^{k_n} q_n(z_n^{(i)} = m)}{k_n}.
\end{align*}

\noindent
On the other hand, for \(m \in [\boldsymbol{M}_n]\), we update \(\omega_{nm}\) by solving:
\begin{align*}
    \omega_{nm}^{t+1} \in \arg \min_{\omega_{nm} \in \mathbb{R}^d} \sum_{i=1}^{k_n} q_{n}^{t+1}(z_n = m) l(h_{\omega_{nm}}(\mathbf{x}_i^{(n)}), y_i^{(n)}).
\end{align*}

\section{Proof of Theorem 1}
\label{section-appendix-thm-1}
\setcounter{theorem}{0}
\begin{theorem}
Under Assumptions 1-3, and given the local model update rule in Eq.~\eqref{eq:12} for the selected neighbors $m \in [\boldsymbol{M}_{n}]$, and assuming that the local models $\omega_{m}$ are bounded, the pFedWN algorithm converges at a rate of $\mathcal{O}(\gamma^T)$ if $\alpha^{2}(2-\alpha)(1-\eta\mu)^{E} \leq 1$ holds. 
\end{theorem}

\begin{proof}
First, we show that the convergence rate is equal to $\mathcal{O}(\gamma^T)$. To this end, we have: 
\begin{equation} \label{eq:21}
\begin{aligned}
    & \mathbb{E} [ ||\omega_{n}^{(t+1)} - \omega_{n}^{*}||^2] \\ 
    &= \mathbb{E} [||\alpha(\omega_{n}^{(t)} - \omega_{n}^{*}) 
    + (1-\alpha)\sum_{m \in [\boldsymbol{M}_n]}\pi_{nm}(\omega_{m}^{(t)}-\omega_{n}^{*})||^2]  \\ 
    &= \alpha^2\mathbb{E}[||\omega_{n}^{(t)} - \omega_{n}^{*}||^2] \\
    &+ (1-\alpha)^2\mathbb{E}[||\sum_{m \in [\boldsymbol{M}_n]}\pi_{nm}(\omega_{m}^{(t)}-\omega_{n}^{*})||^2]\\    &+ 2\alpha(1-\alpha)\mathbb{E}[ \langle \omega_{n}^{(t)} - \omega_{n}^{*}, \sum_{m \in [\boldsymbol{M}_n]}\pi_{nm}(\omega_{m}^{(t)}-\omega_{n}^{*})\rangle ]. 
\end{aligned}
\end{equation}

\noindent 
By bounding the cross term in Eq.~\eqref{eq:21} using Young's inequality, where $\langle a,b \rangle < \frac{c}{2}||a||^2+\frac{1}{2c}||b||^2$ for any $c >0$, we assume $c = \alpha$ and obtain:
\begin{equation}\label{eq:22}
    \begin{aligned}
    & \mathbb{E} [ ||\omega_{n}^{(t+1)} - \omega_{n}^{*}||^2] \\ 
    &\leq \alpha^2\mathbb{E}[||\omega_{n}^{(t)} - \omega_{n}^{*}||^2]+ \alpha^2(1-\alpha)\mathbb{E}[||\omega_{n}^{(t)} - \omega_{n}^{*}||^2]\\
    &+2(1-\alpha)\mathbb{E}[||\sum_{m \in [\boldsymbol{M}_n]}\pi_{nm}(\omega_{m}^{(t)}-\omega_{n}^{*})||^2] \\ 
    &= \alpha^2(2-\alpha)\mathbb{E}[||\omega_{n}^{(t)} - \omega_{n}^{*}||^2] \\
    &+ 2(1-\alpha)\mathbb{E}[||\sum_{m \in [\boldsymbol{M}_n]}\pi_{nm}(\omega_{m}^{(t)}-\omega_{n}^{*})||^2].
    \end{aligned}
\end{equation}

\noindent 
Next, we need to bound $\mathbb{E}[||\omega_{m}^{(t)}-\omega_{n}^{*}||^2]$. First, we have:
\begin{align*}
    \omega_{m}^{(t)}-\omega_{n}^{*} = (\omega_{m}^{(t)}- \omega_{m}^{*})+(\omega_{m}^{*}-\omega_{n}^{*}).
\end{align*}

\noindent 
Using the inequality $||a+b||^2 \leq 2||a||^2 + 2||b||^2$, we get:
\begin{align*}
    ||\omega_{m}^{(t)}-\omega_{n}^{*}||^2 \leq 2||\omega_{m}^{(t)}-\omega_{m}^{*}||^2 + 2||\omega_{m}^{(t)}-\omega_{n}^{*}||^2.
\end{align*}

\noindent
Taking the expectations, we have:
\begin{align*}
    \mathbb{E}[||\omega_{m}^{(t)}-\omega_{n}^{*}||^2 \leq 2\mathbb{E}[||\omega_{m}^{(t)}-\omega_{m}^{*}||^2] + 2||\omega_{m}^{*}-\omega_{n}^{*}||^2,
\end{align*}
where $||\omega_{m}^{*}-\omega_{n}^{*}||$ is deterministic.

\noindent 
For strongly convex functions, the expected distance from the optimum can be written as \cite{li2019convergence, yang2024federated}:
\begin{equation}\label{eq:25}
\begin{aligned}
    &\mathbb{E}[||\omega_{n}^{(t)} - \omega_{n}^{*}||^2] \\
    &\leq (1-\eta\mu)^{E}\mathbb{E}[||\omega_{n}^{(t-1)} - \omega_{n}^{*}||^2] + \frac{\eta\sigma_{n}^2}{\mu}\sum_{k=0}^{E-1}(1-\eta\mu)^{k}.
\end{aligned}
\end{equation}
For neighbor clients, we have:
\begin{equation}\label{eq.26}
    \mathbb{E}[||\omega_{m}^{(t)} - \omega_{m}^{*}||^2] \leq (1-\eta\mu)\mathbb{E}[||\omega_{m}^{(t-1)} - \omega_{m}^{*}||^2] + \frac{\eta\sigma_{m}^2}{\mu}.
\end{equation}

\noindent 
Thus, we plug Eq.~\eqref{eq:25} and Eq.~\eqref{eq.26} back to Eq.~\eqref{eq:22} and obtain:
\begin{align*}
    \mathbb{E} [ ||\omega_{n}^{(t+1)} - \omega_{n}^{*}||^2] \leq \gamma\mathbb{E} [ ||\omega_{n}^{(t)} - \omega_{n}^{*}||^2] + A, 
\end{align*}
where $\gamma = \alpha^{2}(2-\alpha)(1-\eta\mu)^{E}$. In addition, assuming that the neighbor clients' models converge and their variances are bounded, we can consider $A$ as a bounded term or decaying over time, where $A = 4(1-\alpha)^{2}(\sum_{m \in [\boldsymbol{M}_n]}\pi_{nm}^{2}(((1-\eta\mu)\mathbb{E}[||\omega_{m}^{(t)}-\omega_{m}^{*}||^2]+\frac{\eta\sigma_{m}^2}{\mu}+\delta_{nm}))$. Next, we unfold the recursion over $T$ communication rounds:
\begin{align*}
    \mathbb{E} [ ||\omega_{n}^{T} - \omega_{n}^{*}||^2] \leq \gamma^{T}\mathbb{E} [ ||\omega_{n}^{0} - \omega_{n}^{*}||^2] + A\sum_{k=0}^{T-1}\gamma^{k},
\end{align*}
since $\gamma < 1$ and $\gamma^T$ decays exponentially with $T$, and we have $\sum_{k=0}^{T-1}\gamma^{k} \leq \frac{1}{1-\gamma}$. Therefore, we finally obtain:
\begin{align*}
    \mathbb{E} [ ||\omega_{n}^{T} - \omega_{n}^{*}||^2] \leq \mathcal{O}(\gamma^T) + \frac{A}{1-\gamma}.
\end{align*}
As $T \rightarrow \infty$, the first term at RHS vanishes, and the expected error converges to a bound determined by $A$ and $\gamma$.
\end{proof}

\setcounter{lemma}{0}
\section{Proof of useful lemmas}\label{AppendixC}
In this section, we present the proof of useful lemmas.

\begin{lemma}\label{lemma1}
    The model parameter update difference of the target client $n$ can be defined as
    \begin{align*}
        \Delta\omega_{n}^{(t)} &= \omega_{n}^{(t+1)}-\omega_{n}^{(t)} \\
        &= - \alpha\eta E g_{n}^{(t)} + (1-\alpha)\sum_{m}\pi_{nm}(\omega_{m}^{(t)}-\omega_{n}^{(t)}).
    \end{align*}

\end{lemma}

\begin{proof}
    First, the target client performs $E$ local gradient descent steps based on Eq.~\eqref{eq:2}. Therefore, we have: 
    \begin{align*}
        \omega_{n}^{(t,j+1)} = \omega_{n}^{(t,j)} - \eta \nabla f_{n}(\omega_{n}^{(t,j)}), j = 0,1,2,...,E-1.   
    \end{align*}

\noindent 
We can relate $\omega_{n}^{(t,E)}$ to $\omega_{n}^{(t)}$ using the cumulative effect of the local updates. After the local updates, the target client aggregates its model with neighbor models according to Eq.~\eqref{eq:5}. In addition, $\omega_{n}^{(t)}$ is the aggregated model from the previous training round. In the initial case, the aggregated model $\omega_{n}^{(t)} = \omega_{n}^{(t,0)}$. We note that: 
    \begin{equation}\label{eq:23}
        \omega_{n}^{(t,E)} = \omega_{n}^{(t)} - \eta \sum_{j=0}^{E-1}\nabla f_{n}(\omega_{n}^{(t,j)}).   
    \end{equation}
    For simplicity, we define $g_{n}^{(t)}$ where: 
    \begin{align*}
        g_{n}^{(t)} = \frac{1}{E}\sum_{j=0}^{E-1}\nabla f_{n}(\omega_{n}^{(t,j)}).
    \end{align*}

    \noindent
    Then, Eq.~\eqref{eq:23} can be written as:
    \begin{align*}
       \omega_{n}^{(t,E)} = \omega_{n}^{(t)} - \eta E g_{n}^{(t)}. 
    \end{align*}

    \noindent
    The aggregation step in Eq.~\eqref{eq:5} becomes:
    \begin{align*}
        &\omega_{n}^{(t+1)} = \alpha(\omega_{n}^{(t)} - \eta E g_{n}^{(t)}) + (1-\alpha)\sum_{m}\pi_{nm}\omega_{m}^{(t)} \\
        &= \alpha\omega_{n}^{(t)} - \alpha\eta E g_{n}^{(t)} + (1-\alpha)\sum_{m}\pi_{nm}\omega_{m}^{(t)} \\
        &=\alpha\omega_{n}^{(t)} + (1-\alpha)\omega_{n}^{(t)} - \alpha\eta E g_{n}^{(t)} + (1-\alpha)\sum_{m}\pi_{nm}\omega_{m}^{(t)} \\
        &\quad - (1-\alpha)\omega_{n}^{(t)}.
    \end{align*}

    \noindent
    Since $\alpha + (1-\alpha) = 1$, we obtain $\alpha\omega_{n}^{(t)} + (1-\alpha)\omega_{n}^{(t)} = \omega_{n}^{(t)}$. Therefore, we have:
    \begin{align*}
        \omega_{n}^{(t+1)} = \omega_{n}^{(t)}  - \alpha\eta E g_{n}^{(t)} + (1-\alpha)(\sum_{m}\pi_{nm}\omega_{m}^{(t)} - \omega_{n}^{(t)}).
    \end{align*}

    \noindent
     According to $\sum_{m}\pi_{nm} = 1$, we have:
    \begin{align*}
        &\sum_{m}\pi_{nm}\omega_{m}^{(t)} - \omega_{n}^{(t)} = \sum_{m}\pi_{nm}\omega_{m}^{(t)} - (\sum_{m}\pi_{nm})\omega_{n}^{(t)}\\
        &= \sum_{m} \pi_{nm}(\omega_{m}^{(t)} - \omega_{n}^{(t)}).
    \end{align*}

    \noindent
    Thus, the aggregation of target client is simplified as:
    \begin{align*}
        \omega_{n}^{(t+1)} = \omega_{n}^{(t)} - \alpha\eta E g_{n}^{(t)} + (1-\alpha)\sum_{m}\pi_{nm}(\omega_{m}^{(t)}-\omega_{n}^{(t)}).
    \end{align*}
    Therefore, the model update difference can be obtained:
    \begin{align*}
        \Delta\omega_{n}^{(t)} &= \omega_{n}^{(t+1)}-\omega_{n}^{(t)} \\
        &= - \alpha\eta E g_{n}^{(t)} + (1-\alpha)\sum_{m}\pi_{nm}(\omega_{m}^{(t)}-\omega_{n}^{(t)}).
    \end{align*}
\end{proof}

\begin{lemma}\label{lemma2}
    Under Assumption 4  and for $\gamma = \frac{\alpha\eta E}{2} - \frac{1 - \alpha}{2\alpha\eta E}$, $\mathbb{E} \left [ \langle \nabla f_{n}(\omega_{n}^{(t)}, \Delta\omega_{n}^{(t)})\rangle \right ]$ is bounded by

        \begin{align*}
            \mathbb{E} \left [ \langle \nabla f_{n}(\omega_{n}^{(t)}), \Delta\omega_{n}^{(t)}\rangle \right ] &\leq -\gamma \mathbb{E}\left[ ||\nabla f_{n}(\omega_{n}^{t})||^{2}\right]  \\
            & +\frac{\alpha\eta E}{2} \mathbb{E}\left[ ||\nabla f_{n}(\omega_{n}^{t}) - g_{n}^{(t)}||^{2}\right] \\
            &+ \frac{(1-\alpha)\alpha\eta E C}{2}.
        \end{align*}

\end{lemma}
\begin{proof}
    Using Lemma \ref{lemma1}, we can rewrite the model update difference as
    \begin{equation}\label{eq:36}
        \begin{aligned}
            &\mathbb{E} \left [ \langle \nabla f_{n}(\omega_{n}^{(t)}), \Delta\omega_{n}^{(t)}\rangle \right ] = - \alpha\eta E\mathbb{E} \left[ \langle \nabla f_{n}(\omega_{n}^{(t)}), g_{n}^{(t)}\rangle \right] +\\
            &(1-\alpha)\sum_{m}\pi_{nm}\mathbb{E} \left[ \langle \nabla f_{n}(\omega_{n}^{(t)}), \omega_{m}^{(t)}-\omega_{n}^{(t)} \rangle \right].
        \end{aligned}
    \end{equation}

    \noindent
    Based on the linearity of inner product, $\mathbb{E} \left[ \langle \nabla f_{n}(\omega_{n}^{(t)}), g_{n}^{(t)}\rangle \right]$ yields
    \begin{align*}
            \mathbb{E} \left[ \langle \nabla f_{n}(\omega_{n}^{(t)}), g_{n}^{(t)}\rangle \right] &= \frac{1}{2}\mathbb{E}\left[ ||\nabla f_{n}(\omega_{n}^{(t)})||^{2} \right] + \frac{1}{2}\mathbb{E}\left[ ||g_{n}^{(t)}||^{2} \right] \\
            &- \frac{1}{2}\mathbb{E}\left[ ||\nabla f_{n}(\omega_{n}^{(t)}) - g_{n}^{(t)}||^{2} \right].
    \end{align*}
    Since $\mathbb{E}\left[ ||g_{n}^{(t)}||^{2} \right] \geq 0$, we have: 
        \begin{equation}\label{eq:38}
        \begin{aligned}
            \mathbb{E} \left[ \langle \nabla f_{n}(\omega_{n}^{(t)}), g_{n}^{(t)}\rangle \right] &\geq \frac{1}{2}\mathbb{E}\left[ ||\nabla f_{n}(\omega_{n}^{(t)})||^{2} \right] - \\
            & \frac{1}{2}\mathbb{E}\left[ ||\nabla f_{n}(\omega_{n}^{(t)}) - g_{n}^{(t)}||^{2} \right].
        \end{aligned}
    \end{equation}

    \noindent
    Now, using the Cauchy-Schwarz inequality for $\mathbb{E} \left[ \langle \nabla f_{n}(\omega_{n}^{(t)}), \omega_{m}^{(t)}-\omega_{n}^{(t)} \rangle \right]$, we have:
    \begin{equation}\label{eq:39}
        \begin{aligned}
            \mathbb{E}& \left[ \langle \nabla f_{n}(\omega_{n}^{(t)}), \omega_{m}^{(t)}-\omega_{n}^{(t)} \rangle \right] \\
            &\leq \mathbb{E}\left[ ||\nabla f_{n}(\omega_{n}^{(t)})|| \cdot ||\omega_{m}^{(t)}-\omega_{n}^{(t)}|| \right].
        \end{aligned}
    \end{equation}

    \noindent
    Using Young's inequality, we obtain: 
    \begin{equation}\label{eq:40}
        \begin{aligned}
            || \nabla f_{n}(\omega_{n}^{(t)})|| \cdot ||\omega_{m}^{(t)}-\omega_{n}^{(t)}|| &\leq \frac{1}{2\alpha\eta E} || \nabla f_{n}(\omega_{n}^{(t)})||^{2} \\
            &+ \frac{\alpha\eta E}{2} ||\omega_{m}^{(t)}-\omega_{n}^{(t)}||^{2}.
        \end{aligned}
    \end{equation}
  Therefore, we plug Eq.~\eqref{eq:40} back to Eq.~\eqref{eq:39}, combining Assumption 4 and the fact that $\sum_{m}\pi_{nm} = 1$, the second term in RHS of Eq.~\eqref{eq:36} can be represented as:
        \begin{equation}\label{eq:41}
        \begin{aligned}
            &(1-\alpha)\sum_{m}\pi_{nm}\mathbb{E} \left[ \langle \nabla f_{n}(\omega_{n}^{(t)}), \omega_{m}^{(t)}-\omega_{n}^{(t)} \rangle \right] \\
            &\leq \frac{1-\alpha}{2\alpha\eta E} \mathbb{E} \left[ || \nabla f_{n}(\omega_{n}^{(t)})||^{2} \right]\\
            &+ \frac{(1-\alpha)\alpha\eta E}{2}\sum_{m}\pi_{nm} \mathbb{E} \left[||\omega_{m}^{(t)}-\omega_{n}^{(t)}||^{2}\right] \\
            &\leq \frac{1-\alpha}{2\alpha\eta E} \mathbb{E} \left[ || \nabla f_{n}(\omega_{n}^{(t)})||^{2} \right] + \frac{(1-\alpha)\alpha\eta EC}{2}.
        \end{aligned}
    \end{equation}

    \noindent
Using Eq.~\eqref{eq:36} and by plugging back Eq.~\eqref{eq:38} and Eq.~\eqref{eq:41}, we get: 
    \begin{align*}
            \mathbb{E} \left [ \langle \nabla f_{n}(\omega_{n}^{(t)}), \Delta\omega_{n}^{(t)}\rangle \right ] &\leq -\frac{\alpha\eta E }{2} \bigl\{ \mathbb{E}\left[ ||\nabla f_{n}(\omega_{n}^{(t)})||^{2} \right]  \\
            &- \mathbb{E}\left[ ||\nabla f_{n}(\omega_{n}^{(t)}) - g_{n}^{(t)}||^{2} \right] \bigr\} \\
            &+ \frac{1-\alpha}{2\alpha\eta E} \mathbb{E} \left[ || \nabla f_{n}(\omega_{n}^{(t)})||^{2} \right] \\
            &+ \frac{(1-\alpha)\alpha\eta EC}{2}.
    \end{align*}

    \noindent
    By re-arranging the terms with $\mathbb{E} \left[ || \nabla f_{n}(\omega_{n}^{(t)})||^{2} \right]$ and let $\gamma = \frac{\alpha\eta E}{2} - \frac{1 - \alpha}{2\alpha\eta E}$, we conclude the proof to obtain the following bound: 
    \begin{equation}\label{eq:43}
        \begin{aligned}
            \mathbb{E} \left [ \langle \nabla f_{n}(\omega_{n}^{(t)}), \Delta\omega_{n}^{(t)}\rangle \right ] &\leq -\gamma \mathbb{E}\left[ ||\nabla f_{n}(\omega_{n}^{t})||^{2}\right]  \\
            & +\frac{\alpha\eta E}{2} \mathbb{E}\left[ ||\nabla f_{n}(\omega_{n}^{t}) - g_{n}^{(t)}||^{2}\right] \\
            &+ \frac{(1-\alpha)\alpha\eta E C}{2}.
        \end{aligned}
    \end{equation}
    
\end{proof}

\begin{lemma}\label{lemma3}
    For Eq.~\eqref{eq:43} in Lemma \ref{lemma2}, and using Assumptions 1 and 5, $\mathbb{E}\left[ ||\nabla f_{n}(\omega_{n}^{t}) - g_{n}^{(t)}||^{2}\right] $ can be bounded as
        \begin{align*}
            \mathbb{E}\left[ ||\nabla f_{n}(\omega_{n}^{t}) - g_{n}^{(t)}||^{2}\right] \leq L^{2}\eta^{2}E^{2}G^{2}.
        \end{align*}

\end{lemma}

\begin{proof}
    Recall that in Lemma \ref{lemma1}, we defined the function $g_{n}^{(t)} = \frac{1}{E}\sum_{j=0}^{E-1}\nabla f_{n}(\omega_{n}^{(t,j)})$. Thus, we have:
    \begin{align*}
             &\mathbb{E}\left[ ||\nabla f_{n}(\omega_{n}^{t}) - g_{n}^{(t)}||^{2}\right] \\
             &= \frac{1}{E}\sum_{j=0}^{E-1}\mathbb{E}\left[ ||\nabla f_{n}(\omega_{n}^{t}) - \nabla f_{n}(\omega_{n}^{(t,j)})||^{2}\right].        
    \end{align*}

    \noindent
    Under Assumption 1, with the fact that $||\bullet|| \geq 0$, we have: 
    \begin{equation}\label{eq:46}
        \begin{aligned}
            \mathbb{E}\left[||\nabla f_{n}(\omega_{n}^{t}) - \nabla f_{n}(\omega_{n}^{(t,j)})||^{2} \right] \leq L^{2}\mathbb{E}\left[||\omega_{n}^{(t)} - \omega_{n}^{(t,j)}||^{2}\right].
        \end{aligned}
    \end{equation}

    \noindent
    According to Eq.~\eqref{eq:2}, we have:
    \begin{align*} 
        \omega^{(t, j)}_{n} = \omega^{(t)}_{n} - \eta \sum_{k=0}^{j-1}\nabla f_{n}(\omega_{n}^{(t, k)}). 
    \end{align*}

    \noindent
    Thus, we have:
    \begin{equation}\label{eq:48}
        \begin{aligned}
            ||\omega^{(t)}_{n} - \omega^{(t, j)}_{n}|| &= ||\eta \sum_{k=0}^{j-1}\nabla f_{n}(\omega_{n}^{(t, k)})|| \\
            &= \eta ||\sum_{k=0}^{j-1}\nabla f_{n}(\omega_{n}^{(t, k)})|| \\
            &\leq \eta \sum_{k=0}^{j-1}||\nabla f_{n}(\omega_{n}^{(t, k)})|| \\
            &\leq \eta j G \leq \eta EG.
        \end{aligned}
    \end{equation}
    
    \noindent
    Finally, plugging Eq.~\eqref{eq:48} back to Eq.~\eqref{eq:46}, yields
    \begin{equation}
        \begin{aligned}
            \mathbb{E}\left[ ||\nabla f_{n}(\omega_{n}^{t}) - g_{n}^{(t)}||^{2}\right] &\leq L^{2}\mathbb{E}\left[||\omega_{n}^{(t)} - \omega_{n}^{(t,j)}||^{2}\right] \\
            &\leq L^{2}\eta^{2}E^{2}G^{2}.
        \end{aligned}
    \end{equation}

\end{proof}

\begin{lemma}\label{lemma4}
    According to Assumptions 4 and 5, we have
$   
        \mathbb{E} \left [ ||\Delta\omega_{n}^{(t)}||^{2} \right ] \leq (\alpha\eta EG + (1-\alpha)\sqrt{C})^{2}.
$
\end{lemma}

\begin{proof}
    Using Lemma \ref{lemma1}, we have:
    \begin{equation}\label{eq:45}
        \begin{aligned}
            &\mathbb{E} \left [ ||\Delta\omega_{n}^{(t)}||^{2} \right ] \\
            &= \mathbb{E}  \left [ ||- \alpha\eta E g_{n}^{(t)} + (1-\alpha)\sum_{m}\pi_{nm}(\omega_{m}^{(t)}-\omega_{n}^{(t)})||^{2} \right] \\
            &= \mathbb{E}  \left [ ||\alpha\eta E g_{n}^{(t)} - (1-\alpha)\sum_{m}\pi_{nm}(\omega_{m}^{(t)}-\omega_{n}^{(t)})||^{2} \right ] \\
            &= \alpha^{2}\eta^{2}E^{2}\mathbb{E}  \left [  ||g_{n}^{(t)}||^{2}\right]   \\
            &+ (1-\alpha)^{2} \mathbb{E}  \left [  ||\sum_{m}\pi_{nm}(\omega_{m}^{(t)}-\omega_{n}^{(t)})||^{2}\right] \\
            &- 2\alpha\eta E(1-\alpha)\mathbb{E} \left[ \langle g_{n}^{(t)}, \sum_{m}\pi_{nm}(\omega_{m}^{(t)}-\omega_{n}^{(t)})\rangle \right].
        \end{aligned}
    \end{equation}

    \noindent
    Using Assumption 5, we have $\mathbb{E}  \left [  ||g_{n}^{(t)}||^{2}\right] \leq G^{2}$. Therefore, the first term RHS in Eq.~\eqref{eq:45} can be bounded as 
    \begin{equation}\label{eq:61}
        \alpha^{2}\eta^{2}E^{2}\mathbb{E}  \left [  ||g_{n}^{(t)}||^{2}\right] \leq \alpha^{2}\eta^{2}E^{2}G^{2}.
    \end{equation}

    \noindent
    Next, we can also bound the second term in Eq.~\eqref{eq:45} as
    \begin{align*}
        &\mathbb{E}  \left [  ||\sum_{m}\pi_{nm}(\omega_{m}^{(t)}-\omega_{n}^{(t)})||^{2}\right]\\
        &= \mathbb{E}  \left [  ||(\sum_{m}\pi_{nm}v_{m})^{T}\sum_{k}\pi_{nk}v_{k})||\right] ,  \forall m,k \in \mathbf{M}_{n}\\
        &= \sum_{m}\sum_{k}\pi_{nm}\pi_{nk}\mathbb{E}  \left [v_{m}^{T}v_{k}\right],
    \end{align*}
 where $v_{m} = \omega_{m}^{(t)}-\omega_{n}^{(t)}$, $v_{k} = \omega_{k}^{(t)}-\omega_{n}^{(t)}$. When $m = k$, $\mathbb{E}  \left [v_{m}^{T}v_{k}\right] = \mathbb{E}  \left [v_{m}^{2}\right] \leq C$. On the other hand,  if $m \neq k$, we use Cauchy-Schwarz Inequality for expectations, $| \mathbb{E}  \left [v_{m}^{T}v_{k}\right]| \leq \sqrt{\mathbb{E}  \left [||v_{m}||^{2}\right]}\sqrt{\mathbb{E}  \left [||v_{k}||^{2}\right]} \leq C$. Thus, we obtain:
    \begin{equation}\label{eq:63}
        (1-\alpha)^{2} \mathbb{E}  \left [  ||\sum_{m}\pi_{nm}(\omega_{m}^{(t)}-\omega_{n}^{(t)})||^{2}\right] \leq (1-\alpha)^{2} C.
    \end{equation}

    \noindent
    For the last term in Eq.~\eqref{eq:45}, we first notice it is true that  $A = x + y - z \leq x + y + |z| \leq a + b + c, \forall x \leq a, y \leq b, |z| \leq c$. Therefore, under Assumptions 4 and 5, we show that $\left|\mathbb{E} \left[ \langle g_{n}^{(t)}, \sum_{m}\pi_{nm}(\omega_{m}^{(t)}-\omega_{n}^{(t)})\rangle \right]\right|$ is bounded as follows:
    \begin{equation}\label{eq:49}
        \begin{aligned}
            &\left|\mathbb{E} \left[ \langle \nabla g_{n}^{(t)}, \sum_{m}\pi_{nm}(\omega_{m}^{(t)}-\omega_{n}^{(t)})\rangle \right]\right| \\
            &\leq \mathbb{E} \left[  || g_{n}^{(t)}|| \cdot ||\sum_{m}\pi_{nm}(\omega_{m}^{(t)}-\omega_{n}^{(t)})|| \right] \\
            &\leq \sqrt{\mathbb{E} \left[  || g_{n}^{(t)}||^{2}  \right]} \cdot \sqrt{\mathbb{E} \left[  || \sum_{m}\pi_{nm}(\omega_{m}^{(t)}-\omega_{n}^{(t)})||^{2} \right]} \\
            &\leq G\sqrt{C}.
        \end{aligned}
    \end{equation}

    \noindent
    Finally, we plug Eq.~\eqref{eq:61}, Eq.~\eqref{eq:63}, and Eq.~\eqref{eq:49} back to Eq.~\eqref{eq:45} to get: 
    \begin{align*}\label{eq:50}
            &\mathbb{E} \left [ ||\Delta\omega_{n}^{(t)}||^{2} \right ] \\
            &= \alpha^{2}\eta^{2}E^{2}\mathbb{E}  \left [  ||g_{n}^{(t)}||^{2}\right]   \\
            &+ (1-\alpha)^{2} \mathbb{E}  \left [  ||\sum_{m}\pi_{nm}(\omega_{m}^{(t)}-\omega_{n}^{(t)})||^{2}\right] \\
            &- 2\alpha\eta E(1-\alpha)\mathbb{E} \left[ \langle g_{n}^{(t)}, \sum_{m}\pi_{nm}(\omega_{m}^{(t)}-\omega_{n}^{(t)})\rangle \right] \\
            &\leq \alpha^{2}\eta^{2}E^{2}\mathbb{E}  \left [  ||g_{n}^{(t)}||^{2}\right]   \\
            &+ (1-\alpha)^{2} \mathbb{E}  \left [  ||\sum_{m}\pi_{nm}(\omega_{m}^{(t)}-\omega_{n}^{(t)})||^{2}\right] \\
            &+ 2\alpha\eta E(1-\alpha) \left | \mathbb{E}  \left[ \langle g_{n}^{(t)}, \sum_{m}\pi_{nm}(\omega_{m}^{(t)}-\omega_{n}^{(t)})\rangle \right] \right | \\
            &\leq \alpha^{2}\eta^{2}E^{2}G^{2} + (1-\alpha)^{2} C + 2\alpha\eta E (1-\alpha)G\sqrt{C} \\
            &= [\alpha\eta EG +(1-\alpha)\sqrt{C}]^{2}.
    \end{align*}
\end{proof}

\section{Proof of Theorem 2}
\label{sec:appendix-thm-2}
\begin{theorem}
With the local model update rule defined in Eq.~\eqref{eq:12} for the selected neighbors $m \in [\boldsymbol{M}_{n}]$, the target client's model update as outlined in Eq.~\eqref{eq:2}, and the aggregation of the target client's model according to Eq.~\eqref{eq:5}, and considering Assumptions 1, 2, 4 and 5,  we can show that the pFedWN algorithm achieves  $
  \frac{1}{T} \sum_{t=1}^{T} \mathbb{E}[||\nabla f_{n}(\omega_{n}^{(t)})||^{2}] \leq \mathcal{O}(\frac{1}{T})+C',
$  which converges at a rate of $\mathcal{O}(\frac{1}{T})$ for non-convex loss function.
\end{theorem}

\begin{proof}
    \noindent
    Using the $L$-smoothness of $f_{n}$, we have: 
    \begin{align*}
    \begin{aligned}
         \mathbb{E} \left [ f_{n}(\omega_{n}^{(t+1)})\right ]
         &\leq \mathbb{E} \left [ f_{n}(\omega_{n}^{(t)}) \right ] + \mathbb{E} \left [ \langle \nabla f_{n}(\omega_{n}^{(t)}, \Delta\omega_{n}^{(t)})\rangle \right ] \\ 
         &\quad + \frac{L}{2} \mathbb{E} \left [ ||\Delta\omega_{n}^{(t)}||^{2} \right ].
    \end{aligned}
    \end{align*}
According to Lemma \ref{lemma2} and \ref{lemma4}, we have:
    \begin{equation}\label{eq:58}
    \begin{aligned}
         &\mathbb{E} \left [ f_{n}(\omega_{n}^{(t+1)})\right ] \\
         &\leq \mathbb{E} \left [ f_{n}(\omega_{n}^{(t)}) \right ] - \gamma \mathbb{E}\left[ ||\nabla f_{n}(\omega_{n}^{t})||^{2}\right]  \\
         &+ \frac{\alpha\eta E}{2} \mathbb{E}\left[ ||\nabla f_{n}(\omega_{n}^{t}) - g_{n}^{(t)}||^{2}\right] \\
         &+ \frac{(1-\alpha)\alpha\eta E C}{2} + \frac{L}{2}[\alpha\eta EG +(1-\alpha)\sqrt{C}]^{2}.
    \end{aligned}
    \end{equation}
    Recalling that $\mathbb{E}\left[ ||\nabla f_{n}(\omega_{n}^{t}) - g_{n}^{(t)}||^{2}\right]$ is bounded in Lemma \ref{lemma3}, we plug it to Eq.~\eqref{eq:58} to obtain: 
    \begin{equation}\label{eq:59}
        \begin{aligned}
            &\mathbb{E} \left [ f_{n}(\omega_{n}^{(t+1)})\right ] \\
            &\leq \mathbb{E} \left [ f_{n}(\omega_{n}^{(t)}) \right ] - \gamma \mathbb{E}\left[ ||\nabla f_{n}(\omega_{n}^{t})||^{2}\right]  + \frac{\alpha\eta E}{2} L^{2}\eta^{2}E^{2}G^{2} \\
            &+ \frac{(1-\alpha)\alpha\eta E C}{2} + \frac{L}{2}[\alpha\eta EG +(1-\alpha)\sqrt{C}]^{2}. 
        \end{aligned}
    \end{equation}

    \noindent
    Next, we sum both sides of Eq.~\eqref{eq:59} over $t = 1$ to $T$ to get: 
    \begin{align*}
        \begin{aligned}
            &\sum_{t=1}^{T}\mathbb{E} \left [ f_{n}(\omega_{n}^{(t+1)})\right ] \\
            &\leq \sum_{t=1}^{T} \mathbb{E} \left [ f_{n}(\omega_{n}^{(t)}) \right ] - \sum_{t=1}^{T} \gamma \mathbb{E}\left[ ||\nabla f_{n}(\omega_{n}^{t})||^{2}\right]\\
            &+ \frac{\alpha\eta E}{2} T L^{2}\eta^{2}E^{2}G^{2} + \frac{(1-\alpha)\alpha\eta E CT}{2} \\
            &+ \frac{LT}{2}[\alpha\eta EG +(1-\alpha)\sqrt{C}]^{2}.
        \end{aligned}
    \end{align*}

    \noindent
    Moving the first term at RHS to the left, we have: 
    \begin{align*}
        \begin{aligned}
            &\sum_{t=1}^{T}\mathbb{E} \left [ f_{n}(\omega_{n}^{(t+1)})\right ] -\sum_{t=1}^{T} \mathbb{E} \left [ f_{n}(\omega_{n}^{(t)}) \right ]\\
            &\leq  - \sum_{t=1}^{T} \gamma \mathbb{E}\left[ ||\nabla f_{n}(\omega_{n}^{t})||^{2}\right] + \frac{\alpha\eta E}{2} T L^{2}\eta^{2}E^{2}G^{2}\\
            & + \frac{(1-\alpha)\alpha\eta E C T}{2} + \frac{LT}{2}[\alpha\eta EG +(1-\alpha)\sqrt{C}]^{2}, 
        \end{aligned}
    \end{align*}
    which is equal to
    \begin{equation}\label{eq:63}
        \begin{aligned}
            &\mathbb{E} \left [ f_{n}(\omega_{n}^{(T+1)}) - f_{n}(\omega_{n}^{(1)}) \right ]\\
            &\leq  - \sum_{t=1}^{T} \gamma \mathbb{E}\left[ ||\nabla f_{n}(\omega_{n}^{t})||^{2}\right] + \frac{\alpha\eta E}{2} T L^{2}\eta^{2}E^{2}G^{2}\\
            & + \frac{(1-\alpha)\alpha\eta E C T}{2} + \frac{LT}{2}[\alpha\eta EG +(1-\alpha)\sqrt{C}]^{2}.
        \end{aligned}
    \end{equation}

    \noindent
    Assuming $f_{n}(\omega_{n}^{(T+1)}) \geq f_{n}^{*}$ where $f_{n}^{*}$ is the minimum of $f_{n}$, we have: 
    \begin{align*}
        \begin{aligned}
            &f_{n}^{*} - f_{n}(\omega_{n}^{(1)})\\
            &\leq  - \sum_{t=1}^{T} \gamma \mathbb{E}\left[ ||\nabla f_{n}(\omega_{n}^{t})||^{2}\right] + \frac{\alpha\eta E}{2} T L^{2}\eta^{2}E^{2}G^{2}\\
            & + \frac{(1-\alpha)\alpha\eta E C T}{2} + \frac{LT}{2}[\alpha\eta EG +(1-\alpha)\sqrt{C}]^{2}.
        \end{aligned}
    \end{align*}

    \noindent
    By assuming
    \begin{align*}
    \begin{aligned}
            C' &= \frac{1}{\gamma} \bigl [ \frac{\alpha\eta E}{2} T L^{2}\eta^{2}E^{2}G^{2} + \frac{(1-\alpha)\alpha\eta E C T}{2} \\
            &+ \frac{LT}{2}[\alpha\eta EG +(1-\alpha)\sqrt{C}]^{2} \bigr ],
    \end{aligned}
    \end{align*}
    we re-arrange Eq.~\eqref{eq:63} to conclude that
    \begin{align*}
        \begin{aligned}
            \frac{1}{T}\sum_{t=1}^{T} \mathbb{E}\left[ ||\nabla f_{n}(\omega_{n}^{t})||^{2}\right]   &\leq \frac{f_{n}^{*} - f_{n}(\omega_{n}^{(1)})}{\gamma T}  + C' \\
            &= \mathcal{O}(\frac{1}{T})+C',  
        \end{aligned}
    \end{align*}    
    which completes the proof. 
\end{proof}

\bibliographystyle{IEEEtran}
\bibliography{IEEEfull,reference}
\end{document}